\documentclass[runningheads]{llncs}
\usepackage[T1]{fontenc}
\usepackage{graphicx}
\usepackage{booktabs}
\usepackage[misc]{ifsym}
\newcommand{\corr}{(\Letter)}

\usepackage{hyperref}
\usepackage{url}
\usepackage{microtype}
\usepackage{graphicx}
\usepackage{subcaption}
\usepackage{multirow}
\usepackage{booktabs} 
\usepackage{hyperref}

\usepackage{amssymb}
\usepackage{mathtools}
\usepackage[capitalize,noabbrev]{cleveref}
\usepackage{comment}

\usepackage[textsize=tiny]{todonotes}
\usepackage{wrapfig}

\usepackage{booktabs}
\usepackage{multirow}
\usepackage{subcaption}
\usepackage[table]{xcolor}
\usepackage{siunitx}

\begin{document}

\title{Geometric and Information Compression of Representations in Deep Learning}

\titlerunning{Geometric and Information Compression of Representations}

\author{Linara Adilova\inst{1} \corr \and
Henning Petzka\inst{2} \and
Asja Fischer\inst{2} \and
Bernhard C. Geiger\inst{3,4,5}}

\authorrunning{L. Adilova et al.}

\institute{Research Center Trustworthy Data Science and Security, Dortmund\\ \email{linara.adilova@tu-dortmund.de}
\and
Faculty of Computer Science, Ruhr University Bochum \\ \email{henning.petzka/asja.fischer@rub.de}
\and
Signal Processing and Speech Communication Laboratory, Graz University of Technology\\
\and
Know Center Research GmbH\\
\and
Graz Center for Machine Learning\\
\email{geiger@tugraz.at}}

\maketitle              

\begin{abstract}
Deep neural networks transform input data into latent representations that support a wide range of downstream tasks.
These representations can be characterized along information-theoretic and geometric dimensions, but their relationship remains poorly understood.
A central open question is whether low mutual information (MI) between inputs and representations necessarily implies geometrically compressed latent spaces and vice versa.
We investigate this question using class-wise clustering as a measure of geometric compression and theoretically sound MI estimation in conditional entropy bottleneck (CEB) networks and continuous dropout networks.
We evaluate the interplay between MI, geometric compression, and generalization on classification tasks under controlled noise injection schemes.
Our findings show that low MI does not reliably correspond to geometric compression, and that the connection between the two is more nuanced than often assumed.
Indeed, our experiments reveal a negative and nonlinear relationship that can reverse when varying training setup.
Our results put forward a hypothesis that generalization acts as a potential confounder in this connection rather than being their direct consequence.
Dataset: \url{https://doi.org/10.17877/RCTRUST-2026-DWMJTZ}.
Code: \url{https://github.com/link-er/information_geometric_compression}.

\end{abstract}

\section{Introduction}
\label{introduction}

The quest for understanding generalization in deep neural networks (DNNs) inspires many researchers to propose different approaches to this important problem~\cite{keskar2016large,dziugaite2017computing,jiangfantastic,liang2019fisher,petzka2021relative}.
It is intuitively evident that the ability to generalize depends on how the DNN transforms input data into latent representations at hidden layers, i.e., on the properties of latent representations, as well as how consistent these properties when changing from training dataset to the unseen data.
With the rise of foundation models, research on representation properties has gained increasing attention, as their effectiveness heavily depends on the characteristics of the latent space.
\begin{figure}
    \centering
         \begin{subfigure}[b]{0.45\textwidth}
         \centering
         \includegraphics[width=\linewidth]{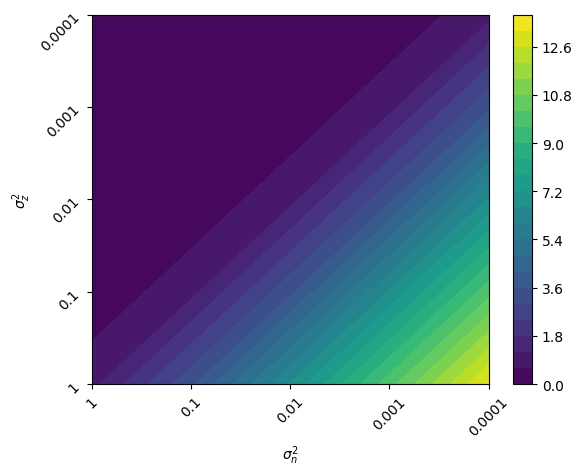}
         \caption{$I(X;Z|Y)$}
     \end{subfigure}
     \hfill
              \begin{subfigure}[b]{0.45\textwidth}
         \centering
         \includegraphics[width=\linewidth]{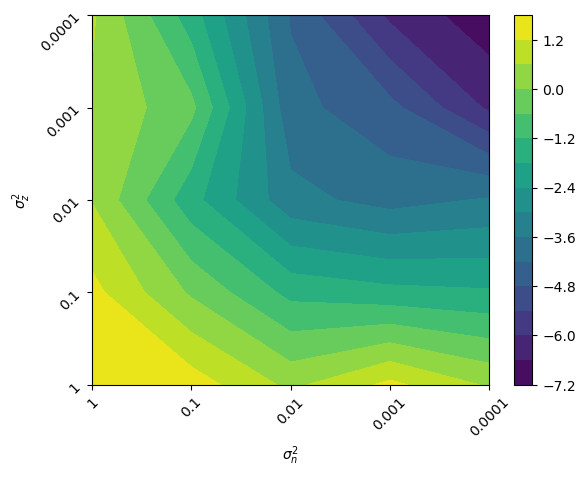}
         \caption{$\log(NC)$ (neural collapse NC1)}
     \end{subfigure}
    \caption{Toy model illustrating the interplay between information-theoretic and geometric compression in terms of neural collapse. 
Data points $x \in \mathcal{X}$ with class labels $y \in \mathcal{Y}$ are encoded into latent representations 
$Z \sim \mathcal{N}(\mu(x), \sigma_n^2 I)$, where class centroids $\mu_y \sim \mathcal{N}(0,I)$ are perturbed by encoder spread $\sigma_z^2$ and noise variance $\sigma_n^2$. 
We vary $\sigma_z^2$ and $\sigma_n^2$ and estimate  geometric compression via neural collapse, averaged over $50$ trials. $I(X;Z|Y)$ is evaluated analytically.
Results show that low MI arises either from strong noise ($\sigma_n^2$ large) or from tightly clustered encodings ($\sigma_z^2$ small), while neural collapse measure indicates compression only in the latter case. 
This demonstrates why low MI and neural collapse do not have to coincide. For the details on the toy example see Appendix~\ref{sec:app:numerical_example}.
}
    \label{fig:theoretical}
\end{figure}
There have been efforts to analyze the geometrical properties of latent representations, for example, by investigating the manifold formed by them.
Following Occam's razor, for classification problems it is natural to seek representations that lie on separated, low-dimensional manifolds associated with different classes.
Low intrinsic dimension as an estimate of the manifold dimension in latent space has been suggested to correlate with good generalization performance~\cite{blier2018description}.
Among the proposed ways to capture geometric compression, one of the most actively studied is neural collapse, which characterizes the class-specific clustering of latent representations and its connection to generalization~\cite{papyan2020prevalence}.

Taking a probabilistic point of view where the input data is drawn from a distribution, one can denote latent representation as a random variable with distribution implicitly described by the DNN.
Again following Occam’s razor, the information bottleneck (IB) theory favors representations that have a large mutual information (MI) with the target but small MI with the input.
Compressed MI with inputs indicates that irrelevant input details are discarded and overfitting is avoided.
While training with the standard cross-entropy loss ensures that the MI between the latent representation and the target is large, it was claimed (and later disputed by \cite{Saxe_IBTheory}) that stochastic gradient descent implicitly reduces the MI between latent representation and input~\cite{Tishby_BlackBox}.

Prior work has suggested a close connection between clustering of representations and information-theoretic compression as defined by the IB principle.
\cite{goldfeld2019estimating} investigated estimates of MI in DNNs with additive Gaussian noise and concluded that  a reduction of MI throughout training is correlated with tightening of the clusters of latent representations.
\cite{Geiger_IPAReview} further argues that many MI estimators are inherently geometric.
In this work, we contribute to this debate and investigate the interplay between geometric and information-theoretic compression using previously unexplored empirical setups, also conjecturing a link to generalization.
In Section~\ref{sec:methodology} we describe the setups used for empirical evaluation: DNNs trained with conditional entropy bottleneck (CEB)~\cite{fischer2020conditional}, i.e., models with additive data-dependent noise, and Gaussian dropout DNNs, i.e., models with multiplicative fixed noise.
For the latent representations of both types of models we estimate the MI with the input (Section~\ref{sec:methodology:MI}) and compute a measure for characterizing class-wise clustering tightness originated from neural collapse (Section~\ref{sec:methodology:NC}).
In contrast to information plane analyses~\cite{Saxe_IBTheory,goldfeld2019estimating} which track MI throughout training, we focus on end-of-training representations.
This choice is deliberate: end-of-training values reflect the stable, converged properties of the learned representation and allow direct comparison across models with different architectures, datasets, and hyperparameters.
Our contributions are:
\begin{enumerate}
    \item We show that DNNs with continuous dropout form a test bed for information-theoretic analyses by proving that the MI between input and latent representations is always finite (Theorem~\ref{thm:finiteMI}), complementing the existing proof for ReLU networks in~\cite{adilovainformation};
    \item We perform large-scale analyses estimating  MI between inputs and representations in state-of-the-art neural networks;
    \item We perform detailed investigations of the connection between information compression and geometric compression in terms of label-specific clustering.
\end{enumerate}

In Section~\ref{sec:experiments}, we show for $8$ architectures (including fully connected DNNs, convolutional DNNs, and transformers) and around $500$ models trained on $5$ different datasets (including vision and NLP tasks) that the connection between geometric and information-theoretic compression is not simply linear.
To illustrate why, we first consider a toy example (Fig.~\ref{fig:theoretical}) in a controlled setting, where latent representations are affected either by noise or by encoder spread.
The example shows that low MI can arise from either source, but geometric compression, in the neural-collapse sense, only occurs for tightly clustered encodings, highlighting why MI and neural collapse do not always coincide.
Motivated by this distinction, we turn to a systematic empirical study.
Studying the interaction between MI and NC empirically requires obtaining models whose representations differ meaningfully along both axes.
We therefore employ regularization schemes during training designed to induce such differences.
However, these representational characteristics are tightly coupled with generalization ability.
As a result, varying MI and NC empirically necessitates considering models that are slightly overfitting or slightly underfitting, while discarding configurations that fail to reach sufficient training accuracy; without this range of generalization behavior, we cannot obtain the variation in MI and NC needed for our analysis.
Across these experiments, we find that MI and NC are negatively but non-linearly correlated.
Moreover, certain training hyperparameter regimes flip this correlation to positive, suggesting an unknown confounder underlying the relationship between geometric and information compression.
We conjecture, but leave for future work to investigate, that this confounder may be related to the generalization ability of the DNN.

\section{Related Work}
\label{sec:related_work}

The connection between representations compression and generalization is explored in multiple research works.
\cite{Tishby_BlackBox} claim that compression of MI between input and latent representation is the implicit regularization of stochastic gradient descent and see it as an explanation of generalization abilities of DNNs.
In its turn, in DNNs with additive noise~\cite{goldfeld2019estimating}, the tightening of class clusters coincides with a reduction in MI between inputs and representations over the course of training.
Similarly, \cite{Patel_LocalRank} show that the local rank of the representations reduces throughout training and connect this to the Gaussian IB~\cite{chechik2003information}.
These works thus suggest a tight link between geometric and information-theoretic compression, which in its turn leads to good generalization. Indeed, for DNNs with additive noise it was shown that neural collapse leads to both grokking and to compression in the sense of IB~\cite{Sakamoto_Grokking}.
Although one might expect that tighter clusters of representations always correspond to reduced MI, this intuition is misleading.
Cluster tightness and MI quantify different properties, and our numerical toy example (Fig.~\ref{fig:theoretical}) demonstrates that clusters can become more compact without any decrease in MI.
More generally, compressibility in a geometric sense is used by algorithms that extract low-dimensional manifolds while preserving relevant information, such as \cite{globerson2003sufficient} or \cite{marx2022estimating}, which indicates that reducing geometrical dimension does not necessarily lead to information reduction.
\cite{Geiger_IPAReview} thus attribute the link between MI and geometric compression to the inherently geometric properties of some MI estimators, while noting that the connection to generalization remains unclear.
In particular, it is established that estimating MI using binning estimators directly reduces the estimate to a measure of geometric compression, cf.~\cite[Fig.~2]{Geiger_IPAReview}.

In contrast to the aforementioned studies that follow MI and clustering throughout training, end-of-training studies focus on representations after convergence of the model.
Whether low MI implies geometric compression is a question about the properties of a converged model, not about training dynamics.
The final MI reflects the information content of the learned representation regardless of the path taken to reach it.
\cite{skean2024does} analyze kernel-based estimates of entropy of language model representations and find that intermediate layers exhibit lower entropy and higher downstream task performance, linking entropy compression to generalization.
\cite{cheng2023bridging} observe a positive correlation between perplexity and the intrinsic dimension of last-token representations in transformers and connect both measures to the performance of finetuning.
None of the previous works embark on estimating MI between inputs and representations, but use some approximations of the amount of information in the representation.

\paragraph{Geometric Characterization of the Representation Space}
A central challenge in studying geometric compression is the lack of a clear definition.
As a result, most works operationalize geometric compression through proxies, such as representations' clustering structure in classification tasks.
Empirically, it has been observed that the penultimate layer representations of well-trained, state-of-the-art classification DNNs collapse such that each class maps to a single point, with these points arranged at the vertices of a simplex equiangular tight frame.
This tight class-wise clustering was termed neural collapse~\cite{papyan2020prevalence} and is a sufficient condition for generalization since it implies linear separability.
However, it is not a necessary condition~\cite{ting2025flatness}.
As an example, DNNs such as RevNets which are reversible~\cite{gomez2017reversible} or Parseval models with orthogonality enforced on weights~\cite{cisse2017parseval} have very limited capabilities for neural collapse, but still achieve state-of-the-art generalization performance.
To date, efforts to relate information-theoretic compression to neural collapse have not yielded conclusive results.
Another line of work indicates that the effective rank of learned representations serves as a generalization predictor in self-supervised and feature-learning settings~\cite{garrido2023rankme},\cite{agrawal2022alpha},\cite{patel2026learning}.

\paragraph{Mutual Information in DNNs}
Estimating MI between input data and representations in deterministic models with continuous distributions is provably vacuous, since this MI is infinite~\cite{amjad2019learning}.
One way to address this is to inject noise into the representations, making them stochastic.
Noise injection schemes differ along two axes: additive vs. multiplicative, and fixed variance vs. input-adaptive variance.
\cite{goldfeld2019estimating} estimated MI by adding Gaussian noise of fixed variance to each neuron output and performed analysis of information flow, as well as comparison to the geometric compression.
The deep variational IB method~\cite{alemi2017deep} and its variants~\cite{fischer2020conditional} add noise with mean and variance learned from the input, i.e., additive adaptive noise.
This setting has not previously been analyzed from the perspective of information compression, a gap our paper addresses.
\cite{adilovainformation} analyzed multiplicative noise from the information compression perspective.
They adopted the adaptive-variance framework of \cite{achille2018information} and additionally introduced multiplicative fixed-variance noise, that is essentially continuous dropout, as a novel setting.
For this fixed-variance case, they proved that continuous dropout guarantees finite MI between inputs and representations in DNNs with ReLU activations.
In this work, we extend that result to analytic activation functions, enabling information-theoretic analysis across modern architectures (GELU, etc.).
We do not pursue multiplicative adaptive-variance noise~\cite{achille2018information} because its key advantage over the fixed-variance approach, which is closed-form MI expressions that avoid numerical estimation, is specific to softplus activations.
For ReLU, GELU, or similar activations, no analogous closed-form expressions exist, so MI would require numerical estimation.

\section{Methodology}
\label{sec:methodology}

We consider a supervised classification setting, that is, DNNs trained on a dataset $\mathcal{S}$ sampled from an unknown distribution $\mathcal{D}$ on $\mathcal{X} \times \mathcal{Y}$.
Here, $\mathcal{X}\subseteq \mathbb{R}^n$ is the space of inputs and $\mathcal{Y}=\{1,\dots,c\}$ is the set of $c$ classes.
We denote by $X$ the (typically continuous) multivariate random variable that describes the input to the DNN, and by $Z$ some representation of $X$ that it produces.
The unknown distribution $\mathcal{D}$ and the DNN induce a distribution of $Z$ in the representation space $\mathcal{Z}\subseteq\mathbb{R}^d$.
Throughout this work, all quantities are evaluated on the converged model after training, not tracked as functions of training epochs.

We consider two types of models in this work: DNNs trained with the Conditional Entropy Bottleneck (CEB)~\cite{fischer2020conditional}, which implements data-dependent additive noise, and models regularized with Gaussian dropout, which implements fixed multiplicative noise.
We chose these models because they complement the existing literature:
Models with fixed additive noise were considered by~\cite{goldfeld2019estimating}, while data-dependent multiplicative noise was examined by~\cite{adilovainformation}.

CEB trained models are slight modification of the variational IB models~\cite{alemi2017deep}, with only difference in the distribution at the bottleneck layer.
The main difference between CEB and variational IB is that the variational marginal for $Z$ is parameterized by the class label.
Variational IB models have demonstrated effectiveness across applications including multi-view learning~\cite{Federici_Multiview}, multi-task learning~\cite{Qian_MTVIB}, and invariant representation learning~\cite{CLUB,moyer2018invariant}, illustrating the practical relevance of this modeling framework.
Models trained with the CEB objective consist of a deterministic decoder and a stochastic encoder that draws $Z$ from a Gaussian distribution with a mean vector and a covariance matrix that depend on the respective input $x \in \mathcal{X}$.
Equivalently, the encoder can be assumed to map the input $x$ deterministically via a learned function $f$ and then add zero-mean Gaussian noise $D(x)$ with a data-dependent covariance, i.e., $Z=f(x)+D(x), D(x) \sim \mathcal{N}(0, \sigma^2(x)I)$.
The decoder $g$ then maps a sampled $Z$ to class probabilities. 
The CEB training objective is given as:
\begin{equation}\label{eq:ceb}
    \mathcal{L}_{CEB} = I(X;Z|Y) + \beta L_{ce}(g(f(X)+D(X)),Y)\enspace,
\end{equation}
where $L_{ce}$ is the cross-entropy loss and $\beta$ trades between classification performance and compression of MI.
CEB-trained models thus explicitly minimize MI between input and latent representation $I(X;Z|Y)$ via increasing $\sigma^2$ - injecting more noise makes $Z$ less informative about $X$.
Conversely, the cross-entropy term aims to decrease $\sigma^2$ since too much noise in the latent representation makes classification inaccurate.
Therefore, for a finite $\beta$, optimization of the loss for every sample settles on some strictly positive $\sigma^2$, which guarantees finiteness of $I(X;Z|Y)$.

Multiplicative noise via dropout represents a widely used form of stochasticity, present in nearly all state-of-the-art architectures.
Although continuous Gaussian dropout has not been as widely adopted as its Bernoulli counterpart, its effect on training is similar and in some aspects more advantageous, as noted in prior work~\cite{srivastava2014dropout}.
However, to compare geometric compression with an information-theoretic one for dropout-regularized networks, we first need to establish that continuous multiplicative noise renders MI between inputs and representations finite and, thus, quantifiable.

\subsection{Guaranteeing Finite MI for Continuous Dropout Models}
In dropout neural networks, latent representations $Z$ equal $f(X)\circ D$, where $\circ$ denotes element-wise multiplication, and $D \sim \mathcal{N}(1, \sigma^2I)$.
For these models it was shown that $I(X;Z)$ is finite if $f$ is a DNN with ReLU activation functions, cf.~Theorem~3.3 and Proposition~3.5 of~\cite{adilovainformation}.
The following theorem extends this result to real analytic activation functions, which include common activation functions such as sigmoid or tanh.
With this result, we effectively show that any DNN with continuous dropout has a provably finite MI between inputs and representations.
This creates the possibility to employ the information-theoretic analysis framework for any state-of-the-art DNN without resorting to purely stochastic or quantized models.

\begin{theorem}[Mutual Information is Finite in Continuous Dropout Networks]\label{thm:finiteMI}
    Consider a non-zero deterministic DNN function $f{:}\ \mathbb{R}^n \rightarrow \mathbb{R}^d$ constructed with finitely many layers, a finite number of neurons per layer, and non-constant real analytic activation functions. Let $X$ be a continuously distributed RV with bounded probability density function $p(x)$ and compact support.\\
    Let $Z = f(X)\circ D(X)$, where $D(X) = (D_1(X),\ldots, D_d(X))$ is (potentially data-dependent) noise with components conditionally independent given $X$ such that all $D_i(X)$ have essentially bounded differential entropy and second moments, i.e., $\mathbb{E} \left [ D_i(X)^2 \right ] \leq M < \infty$ $X$-almost surely for some $M$. Then, $I(X;Z) < \infty$.
\end{theorem}

\paragraph{Proof sketch} In order to prove this theorem for analytical activation functions one needs to provide a proof that $\mathbb{E}\left [\log |f(X)|\right] = \int_K \log |f(x)|\, p(x)\, \mathrm{d}x$ is finite~\cite{adilovainformation}.
Since $p(x)$ is bounded and has compact support, the integral can only diverge to $-\infty$ near zeros of $f$.  
Because $f$ is real analytic, its set of zeros $\mathcal{Z}_f$ in the support of $p(x)$ has a structured form: by the Lojasiewicz Structure Theorem, $\mathcal{Z}_f$ is a finite union of lower-dimensional varieties (of dimension at most $n{-}1$). 
Hence $\mathcal{Z}_f$ has measure zero and the integral is well-defined.  
Near the zeros, the Lojasiewicz inequality for analytic functions ensures that $|f(x)|$ is bounded below by a polynomial in the distance to $\mathcal{Z}_f$: $|f(x)| \;\geq\; C \cdot \mathrm{dist}(x,\mathcal{Z}_f)^q$.
Thus, $\log|f(x)|$ is controlled by $\log(\mathrm{dist}(x,\mathcal{Z}_f))$ near $\mathcal{Z}_f$.  
Finally, by stratifying $\mathcal{Z}_f$ into smooth manifolds $M_j$ of various dimensions and integrating in polar coordinates around each $M_j$, the local integrals reduce to terms of the form $\int_0^\epsilon \log(r)\, r^m\, \mathrm{d}r$, which are finite for any $m \ge 0$. 
Summing over all components gives the result. The full proof is provided in Appendix~\ref{app:proof}.

\paragraph{Remark} Any latent representation that follows the dropout layer also has a finite MI with input due to the data processing inequality.

This theorem serves a dual purpose: it underlies the experiments presented in this paper, and independently, it provides a theoretically sound foundation for future information-theoretic analysis of dropout-regularized networks within a frequentist framework.

\subsection{Estimating Mutual Information}
\label{sec:methodology:MI}

With finite MI between $X$ and $Z$ guaranteed, the challenge is to estimate it from a dataset $\mathcal{S} \sim \mathcal{D}$.

\paragraph{CEB}
For CEB-trained models, the variational bound on $I(X;Z|Y)$ is directly embedded in the training objective, so the MI estimate is available ``for free'' during training.
By definition, the variational formulation of~\eqref{eq:ceb} is
\begin{equation}\label{eq:ceb:variational}
     \min_{q_{Z|Y}} \mathbb{E}\left[KL(e_{Z|X}(\cdot|X)\Vert q_{Z|Y}(\cdot|Y))\right] + \beta L_{ce}(g(f(X)+D(X)),Y) \enspace,
\end{equation}
where the encoder $e_{Z|X}=\mathcal{N}(f(x),\sigma(x))$ is defined by the DNN and where the expectation is taken w.r.t.\ the data distribution $\mathcal{D}$.
In this equation
\begin{multline}
    \mathbb{E}\left[KL(e_{Z|X}(\cdot|X)\Vert q_{Z|Y}(\cdot|Y))\right] = \\
    I(X;Z|Y) + \mathbb{E}\left[KL(p_{Z|Y}(\cdot|Y)\Vert q_{Z|Y}(\cdot|Y))\right]
\end{multline} (cf.~\cite[eq.~(12a)]{Geiger_VIBvsVCEB}).
However, this estimate of MI is only useful if the gap $\mathbb{E}[KL(p_{Z|Y}\,\|\,q_{Z|Y})]$ remains small.
In the CEB setting, $q_{Z|Y}$ is parameterized by a learned backward encoder, which is a separate network that takes the class label $y$ as input and outputs a distribution over $Z$.
This backward encoder is optimized jointly with the forward encoder $e_{Z|X}$, actively fitting $q_{Z|Y}$ to track the true class-conditional marginal $p_{Z|Y}$.
This co-training keeps the gap $\mathbb{E}[KL(p_{Z|Y} \Vert q_{Z|Y})]$ small, ensuring that the variational expression provides a practically tight surrogate for $I(X;Z|Y)$ rather than a loose upper bound.

\paragraph{Dropout} Models trained with Gaussian dropout require an external MI estimator.
Among the published MI estimators with established theoretical guarantees, we selected the Difference of Entropies (DoE) estimator~\cite{mcallester2020formal} as the most suitable for high-dimensional continuously distributed representations; we provide a detailed comparison with MINE~\cite{belghazi2018mine} and CLUB~\cite{cheng2020club} in Appendix~\ref{sssec:mi_compare}.
To estimate MI between input $X$ and representation $Z$, the DoE estimator computes $I(X;Z) = H(Z) - H(Z|X)$ by separately estimating the unconditional entropy $H(Z)$ and the conditional entropy $H(Z|X)$.
In practice, $H(Z)$ is approximated using a parametric logistic distribution with learnable parameters $q_Z$.
The conditional entropy $H(Z|X)$ is approximated using a neural network $q_{Z|X}$ that predicts a distribution over $Z$ conditioned on $X$.
Both terms are estimated via the mean negative log-likelihood (cross-entropy) of observed samples:
\begin{align*}
    H(Z) &\approx -\frac{1}{N}\sum_{i=1}^N \log q_Z(y_i) \enspace, \\
    H(Z|X) &\approx -\frac{1}{N}\sum_{i=1}^N \log q_{Z|X}(z_i|x_i) \enspace.
\end{align*}
During training, we minimize the negative log-likelihoods of both approximations.
Iteratively, this process makes the learned distributions $q_Z$ and $q_{Z|X}$ as close as possible to the true marginal $p_Z$ and conditional $p_{Z|X}$, respectively.
The MI is then obtained as the difference between the two estimates.
A key advantage of DoE is that it provides a stable, finite-sample estimate of MI, particularly when the true MI is large, which often leads to saturated estimates with traditional variational lower bounds~\cite{mcallester2020formal}.

CEB and dropout setups necessarily differ in their approach to MI estimation.
In the CEB setting, a variational upper bound on $I(X;Z|Y)$ is embedded directly in the training objective and in the dropout setting no such training-time bound is available.
The fact that the key empirical finding relating geometric and information compression, i.e., non-linear negative correlation, presented in Section~\ref{sec:experiments} is consistent across both estimations with different bias characteristics strengthens confidence in the result.
A further asymmetry concerns conditioning on the class label $Y$.
In the CEB setup, only $I(X;Z|Y)$ is available, while in the dropout setup estimating $I(X;Z|Y)$ reliably is infeasible.
The number of samples per class is insufficient for any external estimator to converge at the per-class level (see Figure~\ref{fig:doe}(c)).
We therefore estimate the full $I(X;Z)$ instead.
This is justified by the Markov chain $Y \rightarrow X \rightarrow Z$, which gives $I(X;Z|Y) = I(X;Z) - I(Z;Y)$.
Since $I(Z;Y) \leq \log(c)$ where $c$ is the number of classes, and our measured $I(X;Z)$ is several orders of magnitude larger than $\log(c)$, the two quantities induce the same relative ordering across setups - which is the only information required for rank correlation analysis.
When rank correlations are computed with $I(X;Z) - \log(c)$ the results on the training data do not change and maximum deviation on the test data is $0.02$.

\subsection{Measuring Geometric Compression}
\label{sec:methodology:NC}
We characterize geometric compression by quantifying the class-wise clustering structure of representations.
Specifically, we adopt the neural collapse (NC1) statistic~\cite{galanti2021role}, relating within-class variance to between-class centroid distances:
\begin{equation}\label{eq:NC}
    \mathrm{NC} = \frac{1}{c(c-1)} \sum_{i=1}^c \sum_{j=i+1}^c \frac{\mathrm{Var}_i+\mathrm{Var}_j}{\Vert \mu_i-\mu_j\Vert^2} \enspace
\end{equation}
where $\mu_i=\sum_\ell z_\ell/K_i$ and $\mathrm{Var}_i=\sum_\ell \Vert\mu_i-z_\ell\Vert^2/K_i$ are the mean and variance of $K_i$ latent representations sampled from inputs $X$ belonging to class $i$, and where $\Vert\cdot\Vert^2$ is the squared Euclidean norm.
NC is small when clusters are tight and well-separated, indicating strong geometric compression, and large when clusters are diffuse or poorly separated.
This statistic originates from the neural collapse literature but is used here as a measure of clustering quality, as it is more reliable in high-dimensional spaces than alternatives such as the Silhouette score or binned entropy~\cite{galanti2021role,ting2025flatness}.


\subsection{Inducing Variation in Compression}
\label{sec:methodology:regul}

To investigate the relationship between mutual information and class-wise clustering strength, we train models under varying regularization settings.
These settings are chosen to induce different training outcomes while still maintaining reasonable performance.

For CEB-trained models, we vary the $\beta$ coefficient in~\eqref{eq:ceb} to control the tradeoff between MI compression and classification performance.
Small $\beta$ down-weights the cross-entropy term, causing the optimizer to prioritize MI compression at the expense of classification accuracy.
Large $\beta$ does the opposite, focusing optimization on classification while allowing more information to flow through the bottleneck.
Sweeping $\beta$ thus produces a controlled set of models spanning a range of MI compression levels and generalization abilities, providing a testbed for examining how geometric compression behaves as information-theoretic compression varies.

A natural way to obtain models with varying properties under continuous dropout would be to vary the dropout variance.
However, this is not viable in practice since stable training requires a narrow range of dropout values.
Furthermore, varying dropout variance primarily controls the amount of injected noise and thus information flow, which would replicate the CEB setup rather than provide an independent testbed.
Instead, we introduce a geometric regularizer as a term $\lambda( \tanh(\alpha \mathrm{NC}))$, where $\lambda$ controls the strength of regularization, $\alpha$ scales $\mathrm{NC}$ in the effective range of $\tanh(\cdot)$, and $\tanh(\cdot)$ ensures boundedness and saturation of the contribution of the regularization term.
$\alpha$ in the experiments is selected through hyperparameter search.
Changing $\lambda$, we obtain models with different degrees of cluster tightness while keeping the dropout variance and the MI estimation setup fixed.
Since clustering tightness is linked to classification performance, this produces models with a range of generalization abilities as a complementary testbed to the CEB experiments.

\section{Experiments}
\label{sec:experiments}
In this section, we present the experimental setup and evaluate geometric and information compression in two settings: additive adaptive noise in CEB-trained models and multiplicative constant noise in Gaussian dropout models.
For dedicated latent variables $Z$, we quantify clustering using $\mathrm{NC}$ introduced in \eqref{eq:NC} and estimate the MI between inputs and $Z$ using either the variational bound implicit in the CEB loss or the DoE estimator discussed in Section~\ref{sec:methodology:MI}.
We then analyze the correlation between these measures and performance of the models.

Since we expect nonlinear relationships between these quantities, we use the rank correlation instead of the linear (Pearson) correlation.
Moreover, different datasets and architectures lead to different ranges of the considered quantities, and normalization is required.
We first convert the quantities in each experimental setup into ranks, normalize these ranks, and then evaluate the linear correlation of the collection of ranks over all experiments.
For details on the ranking computation see Appendix~\ref{sec:correlation_computation}.
To reduce variance in our correlation estimates, we pooled results across random seeds rather than computing separate per-seed confidence intervals.
To verify the robustness of this choice, we additionally computed rank correlations per seed; these results, reported in the Appendix\ref{sec:per_seed_confidence}, confirm the stability of the estimates.

\subsection{CEB-trained Networks}
We devised four setups of the experiments with CEB.
\begin{wrapfigure}{r}{0.35\textwidth}
     \centering
      \includegraphics[width=1.0\linewidth]{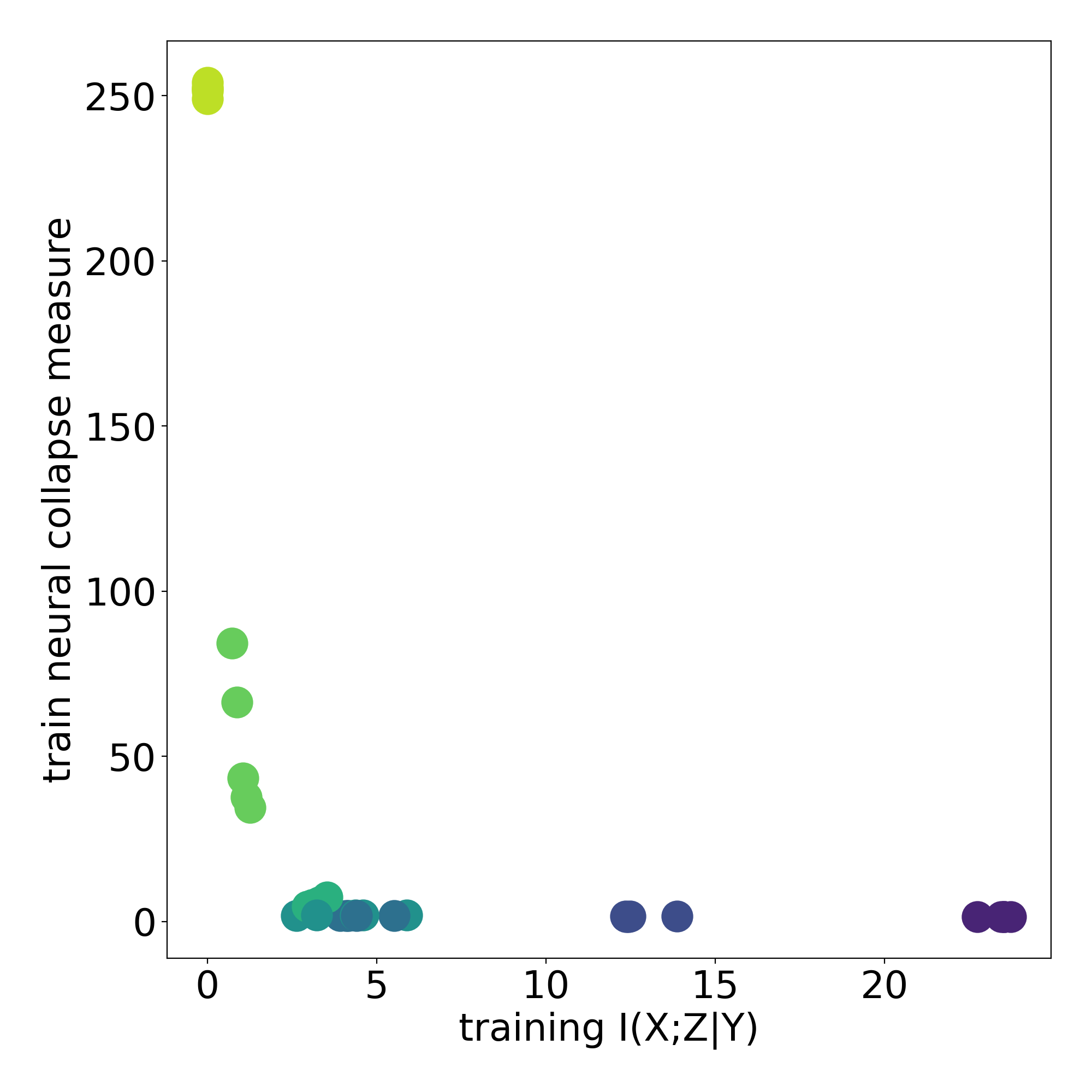}
         \caption{$\mathrm{NC}$ against MI on the train data for CEB (DenseNet121 on CIFAR-100). Strong clustering can correspond to large and small MI, and strong compression of MI can correspond both to weak and strong clustering.}
        \label{fig:mi_nc_ceb}
\end{wrapfigure}
In each setup, we swept $\beta$ over ${1000, 500, 100, 50, 10, 5, 2}$ and trained five models per value with different random seeds.
We observed that very small $\beta$ leads to underfitted models, while getting $\beta$ too large can lead to overfitting.
The four considered setups are as following:
\begin{enumerate}
    \item An MLP (5 hidden layers with 1024 neurons each, ReLU activation functions) as an encoder for FashionMNIST. Dimensionality of $Z$ is selected to be $256$. The decoder is an MLP with one hidden layer of size $1024$ and ReLU.
    \item LeNet5~\cite{lecun1998gradient} encoder for FashionMNIST. Dimensionality of $Z$ is selected to be $64$. The decoder is an MLP with one hidden layer of size $1024$ and ReLU.
    \item WideResNet-28-4~\cite{zagoruyko2016wide} encoder for CIFAR-10. Dimensionality of $Z$ is selected to be $256$. Decoder is an MLP with one hidden layer of size $1024$ and ReLU.
    \item Densenet121~\cite{huang2017densely} encoder for CIFAR-100. Dimensionality of $Z$ is selected to be $256$. Decoder is an MLP without hidden layers.
\end{enumerate}
The backward encoder, i.e., the variational distribution $q_{Z|Y}$, is modeled as a family of Gaussians parameterized by means and variances depending on the class label $y$, therefore in all the setups it is a shallow MLP without any hidden layers.
In CEB-trained models, the latent representation is stochastic by design, which requires sampling in order to evaluate both the training loss and downstream measures: In each setup we drew $8$ samples per data point from the representation distribution during training.
The same sampling procedure was applied at evaluation time, and MI was extracted from the CEB loss.
For details on the training setup, see Appendix~\ref{sec:hyperparameters}.

As an exemplary setup we demonstrate the non-linear correlation between MI and NC in the experiments with DenseNet121 in Fig.~\ref{fig:mi_nc_ceb}.
Large MI can be consistent with strong clustering of the representations, and vice versa.
We refer to an analysis of this phenomenon in Fig.~\ref{fig:theoretical}, where we demonstrate that information-theoretic compression can be achieved by both strong clustering with small $\mathrm{NC}$ and by uninformative, noisy representations with large $\mathrm{NC}$.
At the same time, strong clustering can still lead to large MI values if the encoder noise is sufficiently small.
Thus, the interplay between information-theoretic and geometric compression in a sense of clustering is quite intricate.

In Table~\ref{tab:ceb} we present the correlations for the entire experiment.
In all four of the setups we see a clear monotonic and non-linear correlation between the NC measure and conditional MI $I(X;Z|Y)$, both for test and train data.
In the first group of rank correlations in Table~\ref{tab:ceb}, we demonstrate that larger $\beta$ leads to higher classification accuracy and larger MI.
This is expected from the CEB training objective, as $\beta$ shifts optimization efforts between compression of MI and performance (see eq.~\eqref{eq:ceb}).
At the same time, $\beta$ is negatively correlated with the NC measure, which means that larger $\beta$ leads to clustered representations.
This goes against the common intuition that tighter clusters co-occur with more compressed MI, but is aligned with the conviction that more clustered representations indicate higher accuracy.

\subsection{Gaussian Dropout Neural Networks}
We devised four setups for the experiments with multiplicative noise.
For the experiments we fix the dropout variance found through hyperparameter search that allows to achieve close to the state-of-the-art training results.
We apply the NC regularization described in Section~\ref{sec:methodology:regul} and we vary the regularization strength $\lambda$ to a subset of values in the grid from $-50$ to $50$.
We also include the setup without any regularization ($\lambda=0$).
For each value of $\lambda$, $5$ DNNs were trained with different random seeds.
We employ standard neural architectures, but in each integrate at least one Gaussian dropout layer and add a fully connected layer of dimensionality $128$ or $512$ as the penultimate layer corresponding to the hidden variable $Z$.\footnote{We introduce a layer of a dimensionality less than the usual final fully-connected layers in such architectures in order to speed up MI estimation and $\mathrm{NC}$ computation, but overall the experimental setup is possible without it as well.}
In order to estimate MI between inputs and representation we employ the logistic DoE estimator from \cite{mcallester2020formal} as discussed in Section~\ref{sec:methodology:MI}.
We sample representations from the trained model without switching off the dropout for estimating MI.
We obtain $4$ different representations for each of the inputs and use this data to train the estimator.
For the information on the estimator architecture and its training we refer the reader to the Appendix~\ref{sec:doe_hyperparam}.

The different setups of the experiment are as follows:
\begin{enumerate}
    \item ResNet-18~\cite{he2016deep} trained on CIFAR-10. For this model we added the dropout layer with variance $0.4$ before the fully-connected layer at the end of the architecture and representation layer of dimensionality $128$.
    \item VGG11~\cite{simonyan2014very} with batch normalization trained on SVHN. For VGG11 we added a dropout layer with variance of $0.5$ in the classifier module after the first layer and representation layer of dimensionality $512$.
    \item Densenet121~\cite{huang2017densely} trained on CIFAR-100. In this architecture we replaced the original binary dropout with continuous Gaussian dropout with variance of $0.3$ and representation layer of dimensionality $128$.
    \item MiniBERT~\cite{bhargava2021generalization,DBLP:journals/corr/abs-1908-08962} trained on the AG’s News Corpus dataset~\cite{zhang2015character}. In this architecture we applied Gaussian dropout to the initial embedding layer with variance of $0.6$ and representation layer of dimensionality $512$.
\end{enumerate}
See Appendix~\ref{sec:hyperparameters} for additional details on the training hyper parameters.

We provide the correlations in the setup in Table~\ref{tab:gaus}.
We again observe negative correlation between the MI and the neural collapse measure, showing that tighter clustering and compressed information do not coincide.
The character of correlation is the same as in the CEB experiments.
Differently from CEB experiments, $\lambda$ does not give a direct control over classification accuracy, so the correlation of the measures, both $\mathrm{NC}$ and MI, with accuracy is very weak.
\definecolor{highlightrow}{gray}{0.9}

\begin{table}[!ht]
\centering
\begin{subtable}{0.5\textwidth}
    \centering
    \begin{tabular}{l S[table-format=-1.2] S[table-format=-1.2]}
        \toprule
        & {Train} & {Test} \\
        \midrule
        acc $\|$ $\beta$    &  0.90 &  0.88 \\
        gen $\|$ $\beta$    &  0.21 &  {---} \\
        MI  $\|$ $\beta$    &  0.76 &  0.84 \\
        NC  $\|$ $\beta$    & -0.99 & -0.98 \\
        \cmidrule(lr){1-3}
        \rowcolor{highlightrow}
        \textbf{MI $\|$ NC} & \textbf{-0.75} & \textbf{-0.83} \\
        \cmidrule(lr){1-3}
        acc $\|$ MI         &  0.59 &  0.74 \\
        acc $\|$ NC         & -0.89 & -0.88 \\
        gen $\|$ MI         &  0.33 &  {---} \\
        gen $\|$ NC         & -0.21 &  {---} \\
        \bottomrule
    \end{tabular}
    \caption{CEB ($\beta$ sweep).}
    \label{tab:ceb}
\end{subtable}%
\hfill
\begin{subtable}{0.5\textwidth}
    \centering
    \begin{tabular}{l S[table-format=-1.2] S[table-format=-1.2]}
        \toprule
        & {Train} & {Test} \\
        \midrule
        acc $\|$ $\lambda$  & -0.02 & -0.30 \\
        gen $\|$ $\lambda$  &  0.09 &  {---} \\
        MI  $\|$ $\lambda$  & -0.48 & -0.25 \\
        NC  $\|$ $\lambda$  &  0.96 &  0.94 \\
        \cmidrule(lr){1-3}
        \rowcolor{highlightrow}
        \textbf{MI $\|$ NC} & \textbf{-0.46} & \textbf{-0.24} \\
        \cmidrule(lr){1-3}
        acc $\|$ MI         & -0.17 &  0.07 \\
        acc $\|$ NC         & -0.05 & -0.34 \\
        gen $\|$ MI         &  0.09 &  {---} \\
        gen $\|$ NC         &  0.07 &  {---} \\
        \bottomrule
    \end{tabular}
    \caption{Gaussian dropout ($\lambda$ sweep).}
    \label{tab:gaus}
\end{subtable}
\caption{Rank correlations between MI, NC, accuracy (\emph{acc}), and generalization gap (\emph{gen}) across all experiments, computed over ranked and normalized values pooled across all network-dataset setups. The upper group shows correlations with the control hyperparameter, the lower group shows correlations with model performance, and the highlighted row shows the central MI-NC relationship. Correlations are reported separately for train and test data to show that the MI--NC relationship holds on unseen data. The generalization gap (\emph{gen}) is a scalar quantity defined as the difference between train and test accuracy, and is therefore only reported in the train column.}
\label{tab:correlations}
\end{table}

\subsection{Observations}
The rank correlations in Tables~\ref{tab:ceb}, \ref{tab:gaus} in all experiments consistently indicate a negative relationship between class-wise clustering and MI compression.
However, in the Gaussian dropout setup, this trend is weaker, as some training configurations showed little correlation or even a positive one.
We demonstrate examples of different setups in Fig.~\ref{fig:observations}.
Here we compare three sets (each point is different random seed and different $\lambda)$: CIFAR-100 and two VGG11 configurations obtained with different values of the training hyperparameter $\alpha$ (see  Sec.~\ref{sec:methodology:regul}).
These settings produce models with distinct generalization behavior, as reflected in their train/test accuracy (top rows).
The MI–NC correlation changes sign across the three regimes ($-0.54$, $0.27$, $0.37$), indicating that the MI–NC relationship is not invariant across setups, and the two quantities may instead be confounded, consistent with Reichenbach's principle of common cause~\cite{reichenbach1991direction}.
In particular, the setup with the most stable generalization behavior (second column) is also where the MI–NC correlation is closest to vanishing, suggesting that generalization itself may act as a confounder in the MI–NC relationship.
While a closer investigation is left for future work, these results provide further evidence for the complex interplay between geometry and information in the representations.

The contrast in NC–accuracy correlation between Tables~\ref{tab:ceb} and \ref{tab:gaus} also reflects a structural difference between the two regularization mechanisms.
In CEB, the $\beta$ sweep is monotonic, moving models from underfitting toward fitting along a single direction, and underfitting reliably produces a spike in NC, reflecting scattered, weakly clustered representations.
The Dropout sweep is structurally different: the regularization coefficient ranges from negative to positive, enforcing and discouraging clustering within the same setup.
Consequently, neither NC nor the regularization strength itself correlates cleanly with accuracy or generalization in this regime.
This is consistent with prior findings~\cite{ting2025flatness} that NC is not necessary for generalization, and is visible in Fig.~\ref{fig:observations}, where higher accuracies correspond to less clustering.
We also note that the rank correlations in Tables~\ref{tab:ceb}/\ref{tab:gaus} can mask an important scale difference: in CEB, NC ranges up to $150$ while MI is tightly bounded; in Dropout, NC is bounded by $6$ while MI can be three orders of magnitude larger — a difference in dynamic range that likely sharpens the contrast between the two settings.
\begin{figure}[!ht]
  \centering
  
  \begin{subfigure}{0.3\textwidth}
    \includegraphics[width=\linewidth]{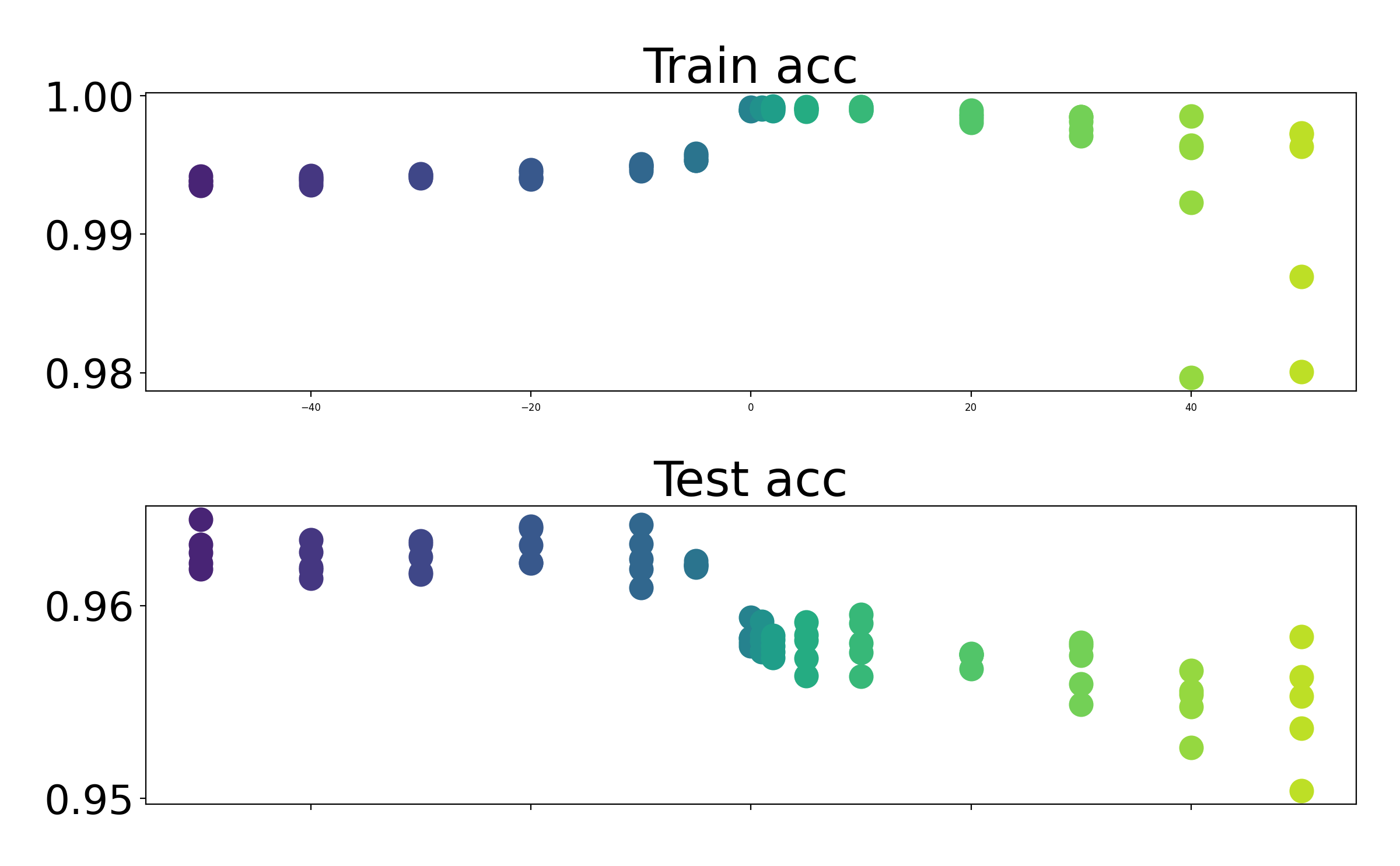}
  \end{subfigure}
  \hfill
  \begin{subfigure}{0.3\textwidth}
    \includegraphics[width=\linewidth]{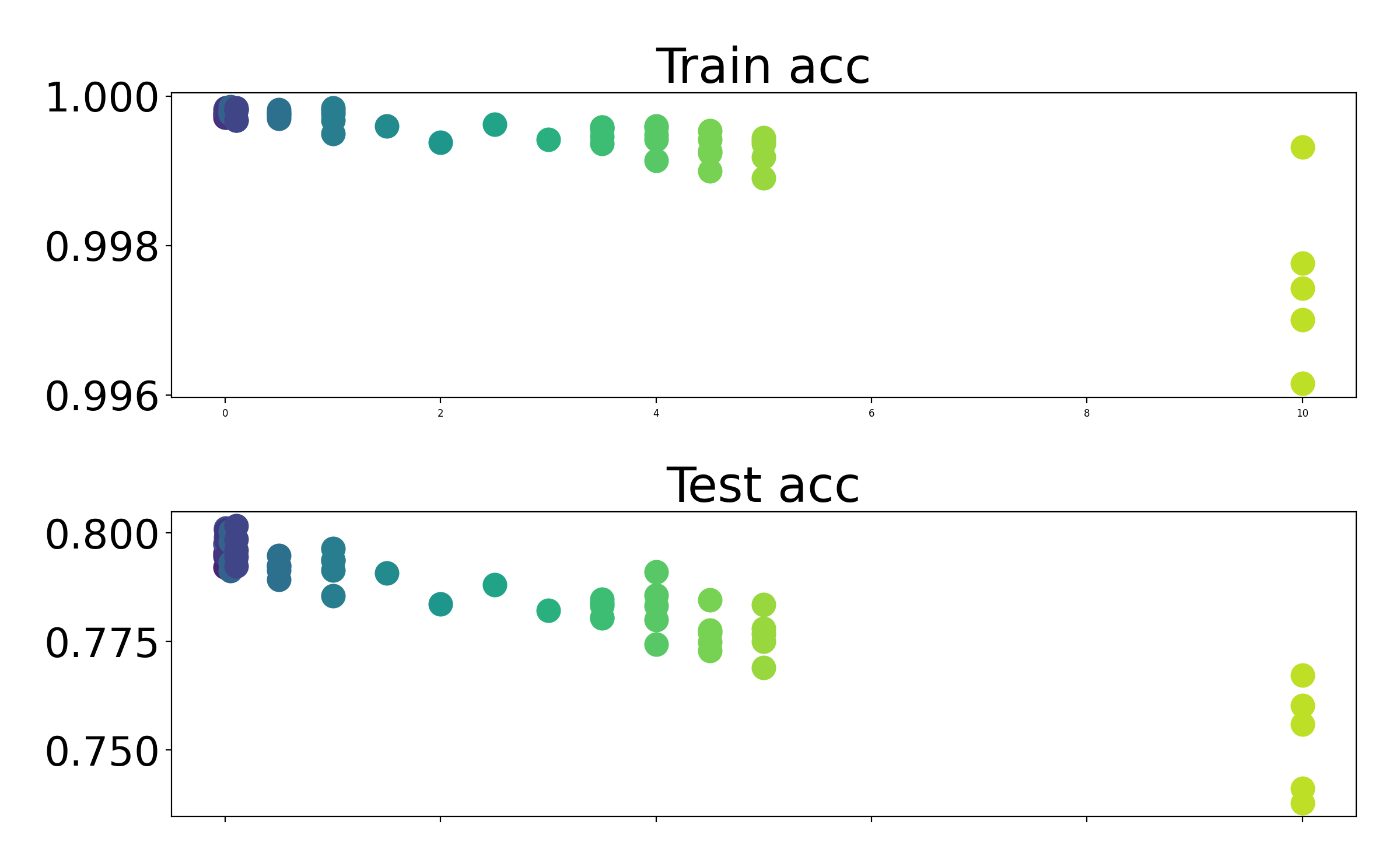}
  \end{subfigure}
  \hfill
  \begin{subfigure}{0.3\textwidth}
    \includegraphics[width=\linewidth]{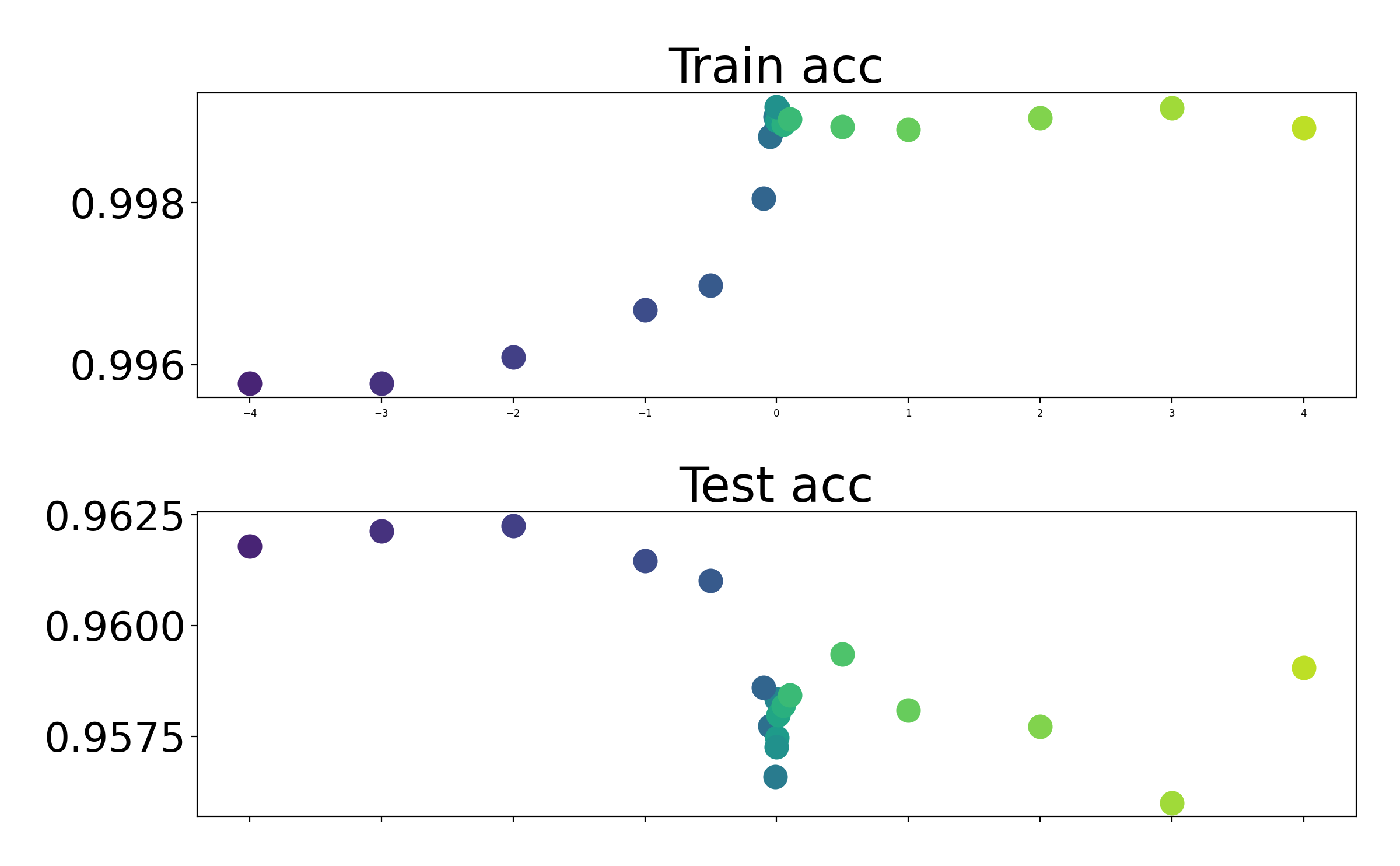}
  \end{subfigure}
  
  \vspace{0.5cm} 
  \begin{subfigure}{0.3\textwidth}
    \includegraphics[width=\linewidth]{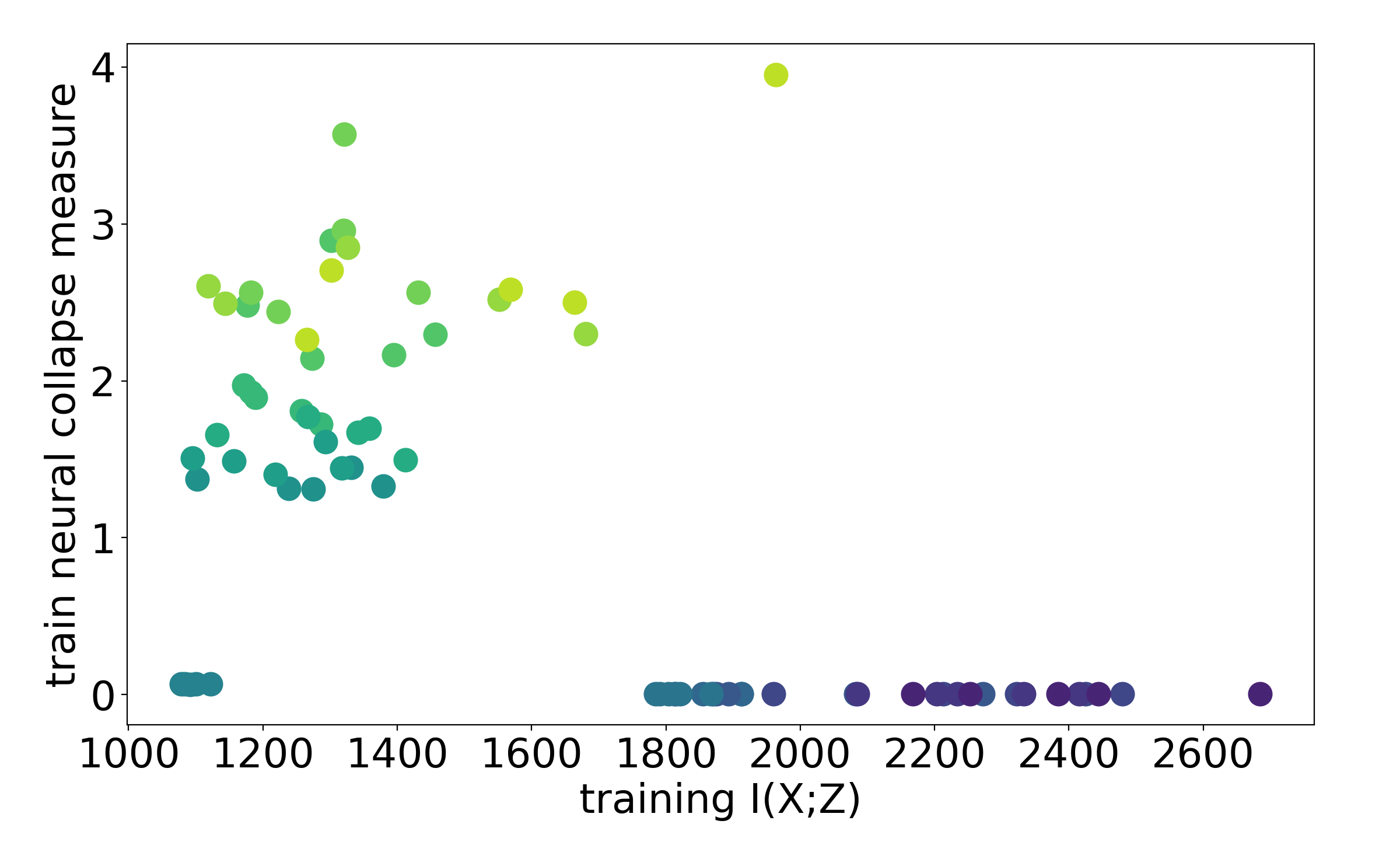}
    \caption{VGG11 on SVHN, $\alpha = 0.5$ (see  Sec.~\ref{sec:methodology:regul})}
  \end{subfigure}
  \hfill
  \begin{subfigure}{0.3\textwidth}
    \includegraphics[width=\linewidth]{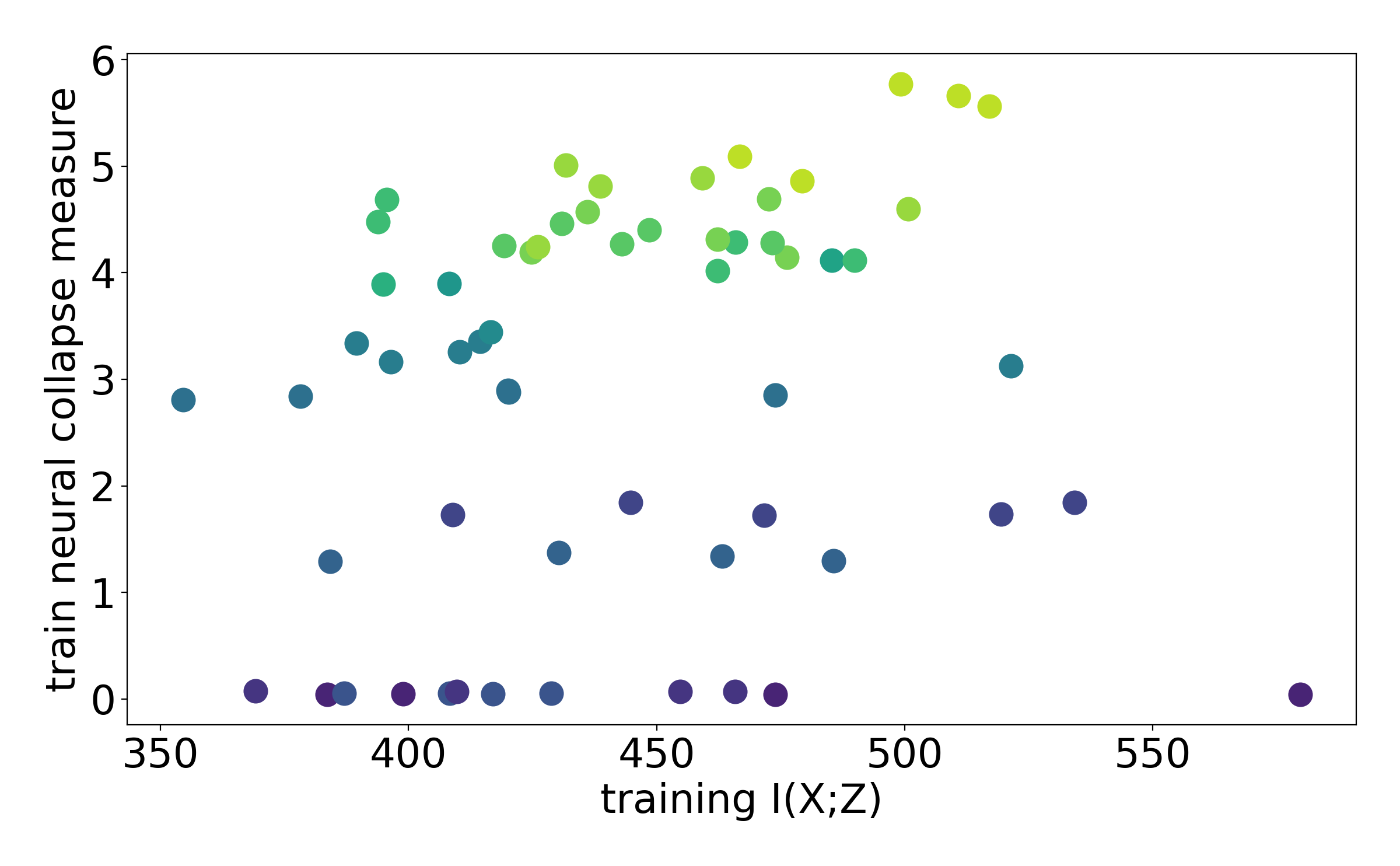}
    \caption{DenseNet121 on CIFAR-100}
  \end{subfigure}
  \hfill
  \begin{subfigure}{0.3\textwidth}
    \includegraphics[width=\linewidth]{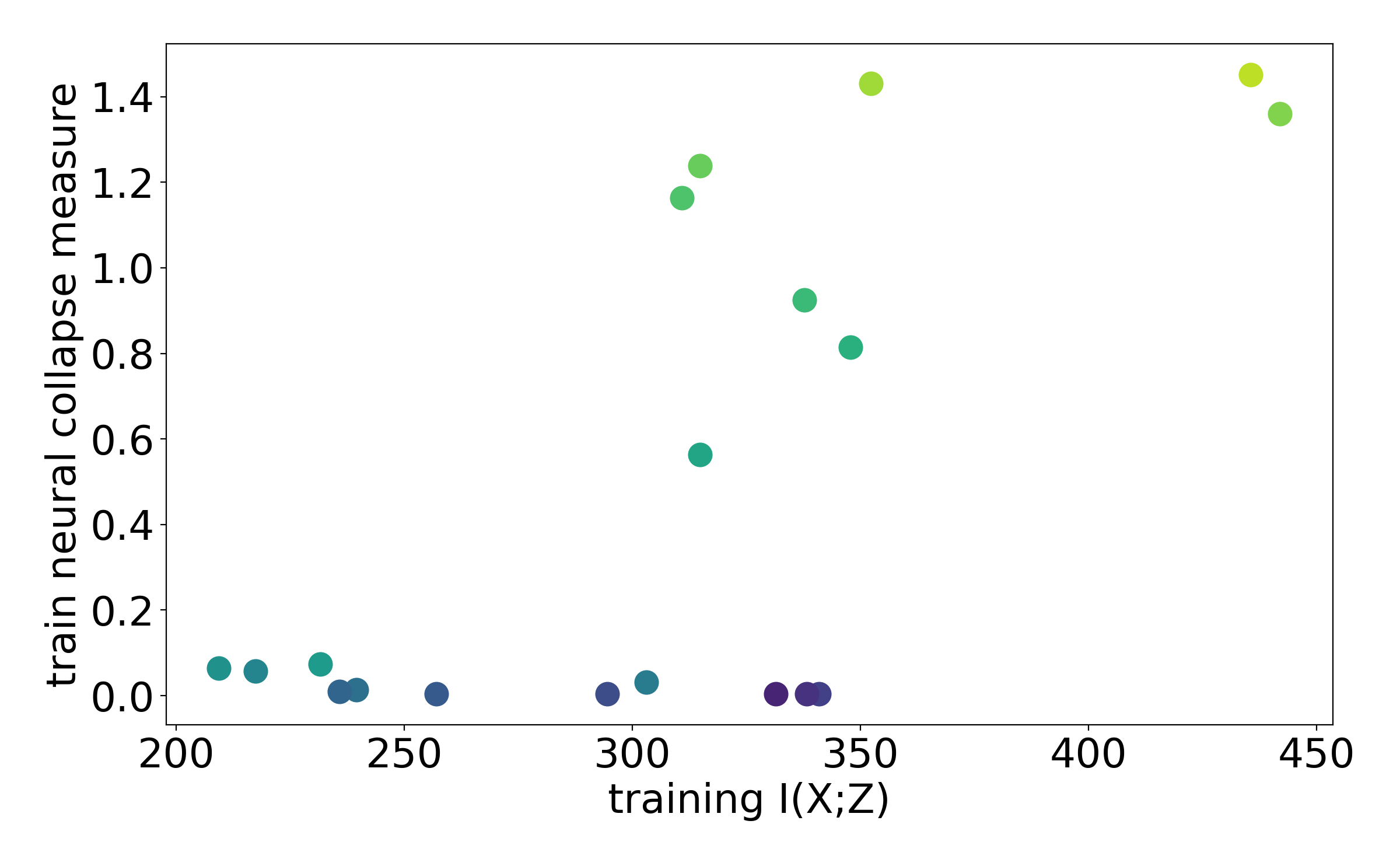}
    \caption{VGG11 on SVHN, $\alpha = 0.01$ (see  Sec.~\ref{sec:methodology:regul})}
  \end{subfigure}

  \caption{Sign of MC-NI correlation is not invariant under the change of the training outcome (top two rows): underfitting (a), normal (b), overfitting (c). The sign of the MI-NC correlation (bottom row) reverses across regimes ($-0.54, 0.27, 0.37$), indicating that generalization might act a confounder.}
  \label{fig:observations}
\end{figure}
%

\section{Discussion and Conclusion}
\label{discussion}

With this paper we contribute to the debate about the connection between geometric compression in terms of clustering of representations and information-theoretic compression in terms of the MI between latent representations and inputs.
Previous research has claimed a positive (potentially non-linear) correlation between both types of compression: MI is small if latent representations are clustered, cf.~\cite{goldfeld2019estimating}, and MI estimation is inherently geometric in many cases, cf.~\cite{Geiger_IPAReview}.
Across two independent experimental frameworks we find a consistent negative nonlinear correlation between MI and class-wise clustering.
This direction of correlation is opposite to what prior work suggested.
Moreover, the sign of the correlation is not fixed: it appears reversed by changing hyperparameters (e.g., $\alpha$ in Fig.~\ref{fig:observations}) affecting generalization ability, which indicates the presence of a confounder rather than a direct relationship.

These findings have concrete implications.
MI and geometric compression should not be used interchangeably as proxies for representation quality.
Also, claims about generalization based on either measure alone should be treated with caution.
The theoretical result extending finite MI guarantees to DNNs with multiplicative noise and arbitrary analytic activation functions opens the door to information-theoretic analysis of modern architectures, which we hope will support more rigorous future investigations.
A further caveat concerns the empirical estimation of MI itself: in high-dimensional spaces and for complex, non-Gaussian distributions, MI estimates can be highly sensitive to the choice of estimator.
Our conclusions are therefore conditioned on the estimators used in this work, and a systematic comparison across estimators remains an important direction for validating the robustness of the observed MI–NC relationship.
An important direction for future work is a proper causal analysis of the dependency structure among MI, geometric compression, and generalization — moving from the correlational evidence presented here to a mechanistic understanding of how these quantities interact during and after training.


%
%
%

\section*{Acknowledgments}
The work of Bernhard C. Geiger was partially funded by the Austrian Science Fund (FWF) under grant 10.55776/PAT7753623 and by the European Union’s HORIZON Research and Innovation Programme under grant agreement No 101120657, project ENFIELD (European Lighthouse to Manifest Trustworthy and Green AI). Know Center is a COMET competence center that is financed by the Austrian Federal Ministry of Innovation, Mobility and Infrastructure (BMIMI), the Austrian Federal Ministry of Economy, Energy and Tourism (BMWET), the Province of Styria, the Steirische Wirtschaftsförderungsgesellschaft m.b.H. (SFG), the Vienna business agency and the Standortagentur Tirol. The COMET programme is managed by the Austrian Research Promotion Agency FFG.

H.\ Petzka and A.\ Fischer are funded by the Deutsche Forschungsgemeinschaft (DFG, German Research Foundation) under Germany’s Excellence Strategy - EXC 2092 CASA - 390781972 and by the Ministry of Culture and Science of North Rhine-Westphalia as part of the Lamarr Fellow Network.

Adilova Linara was partially supported by the grant on ''Information-theoretic analysis of generalization in deep learning`` from the European Union, via the oc1-2024-TES-01-04 issued and implemented by the ENFIELD project, under the grant agreement No 101120657. 
Computational resources were provided by the German AI Service Center WestAI.

\bibliographystyle{splncs04}
\bibliography{references}

@inproceedings{adilovainformation,
  title={Information Plane Analysis for Dropout Neural Networks},
  author={Adilova, Linara and Geiger, Bernhard C and Fischer, Asja},
  year = {2023},
  booktitle={The Eleventh International Conference on Learning Representations}
}

@inproceedings{cheng2023bridging,
  title={Bridging Information-Theoretic and Geometric Compression in Language Models},
  author={Cheng, Emily and Kervadec, Corentin and Baroni, Marco},
  booktitle={Proceedings of the 2023 Conference on Empirical Methods in Natural Language Processing},
  pages={12397--12420},
  year={2023}
}

@inproceedings{goldfeld2019estimating,
  title={Estimating Information Flow in Deep Neural Networks},
  author={Goldfeld, Ziv},
  booktitle={International Conference on Machine Learning},
  year={2019}
}

@article{achille2018information,
  title={Information dropout: Learning optimal representations through noisy computation},
  author={Achille, Alessandro and Soatto, Stefano},
  journal={IEEE transactions on pattern analysis and machine intelligence},
  volume={40},
  number={12},
  pages={2897--2905},
  year={2018},
  publisher={IEEE}
}

@inproceedings{alemi2017deep,
  title={Deep Variational Information Bottleneck},
  author={Alemi, Alexander A and Fischer, Ian and Dillon, Joshua V and Murphy, Kevin},
  booktitle={International Conference on Learning Representations},
  year={2017}
}

@article{fischer2020conditional,
  title={The conditional entropy bottleneck},
  author={Fischer, Ian},
  journal={Entropy},
  volume={22},
  number={9},
  pages={999},
  year={2020},
  publisher={MDPI}
}

@article{papyan2020prevalence,
  title={Prevalence of neural collapse during the terminal phase of deep learning training},
  author={Papyan, Vardan and Han, XY and Donoho, David L},
  journal={Proceedings of the National Academy of Sciences},
  volume={117},
  number={40},
  pages={24652--24663},
  year={2020},
  publisher={National Acad Sciences}
}

@article{blier2018description,
  title={The description length of deep learning models},
  author={Blier, L{\'e}onard and Ollivier, Yann},
  journal={Advances in Neural Information Processing Systems},
  volume={31},
  year={2018}
}

@inproceedings{chechik2003information,
  title={Information bottleneck for {G}aussian variables},
  author={Chechik, Gal and Globerson, Amir and Tishby, Naftali and Weiss, Yair},
  booktitle={Advances in Neural Information Processing Systems},
  volume={16},
  year={2003}
}

@inproceedings{skean2024does,
  title={Does Representation Matter? Exploring Intermediate Layers in Large Language Models},
  author={Skean, Oscar and Arefin, Md Rifat and Shwartz-Ziv, Ravid},
  booktitle={Workshop on Machine Learning and Compression @ NeurIPS},
  year={2024}
}

@misc{Tishby_BlackBox,
  author = {R. Shwartz-Ziv and N. Tishby},
  title = {Opening the Black Box of Deep Neural Networks via Information},
  year = {2017},
  month = apr,
  howpublished = {{\tt arXiv:1703.00810v3 [cs.LG]}},
}

@ARTICLE{Geiger_IPAReview,
  author = {Geiger, Bernhard Claus},
  title = {On Information Plane Analyses of Neural Network Classifiers -- A Review},
  journal={IEEE Transactions on Neural Networks and Learning Systems},
  year = {2022},
  month=dec,
  pages={7039-7051},
  number={12},
  volume={33},
}

@MISC{Qian_MTVIB,
    title={Multi-Task Variational Information Bottleneck},
    author={Weizhu Qian and Bowei Chen and Yichao Zhang and Guanghui Wen and Franck Gechter},
    year={2020},
    month=jul,
    howpublished = {{\tt arXiv:2007.00339 [cs.LG]}},
}

@inproceedings{moyer2018invariant,
    author = {Moyer, Daniel and Gao, Shuyang and Brekelmans, Rob and Galstyan, Aram and Ver Steeg, Greg},
    booktitle = {Advances in Neural Information Processing Systems},
    pages = {},
    title = {Invariant Representations without Adversarial Training},
    volume = {31},
    year = {2018}
}

@ARTICLE{CLUB,
  author={Razeghi, Behrooz and Calmon, Flavio P. and Gunduz, Deniz and Voloshynovskiy, Slava},
  journal={IEEE Transactions on Information Forensics and Security}, 
  title={Bottlenecks {CLUB}: Unifying Information-Theoretic Trade-Offs Among Complexity, Leakage, and Utility}, 
  year={2023},
  volume={18},
  number={},
  pages={2060-2075},
  doi={10.1109/TIFS.2023.3262112}}

@inproceedings{Federici_Multiview,
  title={Learning Robust Representations via Multi-View Information Bottleneck},
  author={Marco Federici and Anjan Dutta and Patrick Forr\'e and Nate Kushman and Zeynep Akata},
  month=apr,
  address={virtual},
  booktitle={Proc. Int. Conf. on Learning Representations (ICLR)},
  howpublished={{\tt arXiv:2002.07017 [cs.LG]}},
  year={2020}
}

@misc{Sakamoto_Grokking,
    title={Explaining Grokking and Information Bottleneck through Neural Collapse Emergence},
    author={Keitaro Sakamoto and Issei Sato},
    year = {2025},
    month=sep,
    note = {{\tt arXiv:2509.20829}}
}

@INPROCEEDINGS{Saxe_IBTheory,
title={On the Information Bottleneck Theory of Deep Learning},
author={A. M. Saxe and Y. Bansal and J. Dapello and M. Advani and A. Kolchinsky and B. D. Tracey and D. D. Cox},
booktitle={Proc. International Conference on Learning Representations (ICLR)},
month=may,
address={Vancouver},
year={2018},
}

@inproceedings{Patel_LocalRank,
  title={Learning to Compress: Local Rank and Information Compression in Deep Neural Networks},
  author={Patel, Niket and Shwartz-Ziv, Ravid},
  booktitle={Workshop on Machine Learning and Compression @ NeurIPS},
  year={2024}
}

@inproceedings{galanti2021role,
  title={On the Role of Neural Collapse in Transfer Learning},
  author={Galanti, Tomer and Gy{\"o}rgy, Andr{\'a}s and Hutter, Marcus},
  booktitle={International Conference on Learning Representations},
  year={2021}
}

@article{globerson2003sufficient,
  title={Sufficient dimensionality reduction},
  author={Globerson, Amir and Tishby, Naftali},
  journal={Journal of Machine Learning Research},
  volume={3},
  number={Mar},
  pages={1307--1331},
  year={2003}
}

@ARTICLE{Geiger_VIBvsVCEB,
  author = {Bernhard C. Geiger and Ian S. Fischer},
  title = {A Comparison of Variational Bounds for the Information Bottleneck Functional},
  journal = {Entropy},
  volume = {22},
  number ={11},
  month=nov,
  pages={1229},
  year = {2020},
  note ={open-access}
}

@inproceedings{marx2022estimating,
  title={Estimating Mutual Information via Geodesic k NN},
  author={Marx, Alexander and Fischer, Jonas},
  booktitle={Proceedings of the 2022 SIAM International Conference on Data Mining (SDM)},
  pages={415--423},
  year={2022},
  organization={SIAM}
}

@inproceedings{cisse2017parseval,
  title={Parseval networks: Improving robustness to adversarial examples},
  author={Cisse, Moustapha and Bojanowski, Piotr and Grave, Edouard and Dauphin, Yann and Usunier, Nicolas},
  booktitle={International conference on machine learning},
  pages={854--863},
  year={2017},
  organization={PMLR}
}

@article{gomez2017reversible,
  title={The reversible residual network: Backpropagation without storing activations},
  author={Gomez, Aidan N and Ren, Mengye and Urtasun, Raquel and Grosse, Roger B},
  journal={Advances in neural information processing systems},
  volume={30},
  year={2017}
}

@inproceedings{zagoruyko2016wide,
  title={Wide Residual Networks},
  author={Zagoruyko, Sergey and Komodakis, Nikos},
  booktitle={British Machine Vision Conference 2016},
  year={2016},
  organization={British Machine Vision Association}
}

@inproceedings{huang2017densely,
  title={Densely connected convolutional networks},
  author={Huang, Gao and Liu, Zhuang and Van Der Maaten, Laurens and Weinberger, Kilian Q},
  booktitle={Proceedings of the IEEE conference on computer vision and pattern recognition},
  pages={4700--4708},
  year={2017}
}

@article{amjad2019learning,
  title={Learning representations for neural network-based classification using the information bottleneck principle},
  author={Amjad, Rana Ali and Geiger, Bernhard C},
  journal={IEEE Transactions on Pattern Analysis and Machine Intelligence},
  volume={42},
  number={9},
  pages={2225--2239},
  year={2019},
  publisher={IEEE}
}

@inproceedings{mcallester2020formal,
  title={Formal limitations on the measurement of mutual information},
  author={McAllester, David and Stratos, Karl},
  booktitle={International Conference on Artificial Intelligence and Statistics},
  pages={875--884},
  year={2020},
  organization={PMLR}
}

@article{lecun1998gradient,
  title={Gradient-based learning applied to document recognition},
  author={LeCun, Yann and Bottou, L{\'e}on and Bengio, Yoshua and Haffner, Patrick},
  journal={Proceedings of the IEEE},
  volume={86},
  number={11},
  pages={2278--2324},
  year={1998},
  publisher={IEEE}
}

@inproceedings{he2016deep,
  title={Deep residual learning for image recognition},
  author={He, Kaiming and Zhang, Xiangyu and Ren, Shaoqing and Sun, Jian},
  booktitle={Proceedings of the IEEE conference on computer vision and pattern recognition},
  pages={770--778},
  year={2016}
}

@article{simonyan2014very,
  title={Very deep convolutional networks for large-scale image recognition},
  author={Simonyan, Karen},
  journal={arXiv preprint arXiv:1409.1556},
  year={2014}
}

@article{dziugaite2017computing,
  title={Computing nonvacuous generalization bounds for deep (stochastic) neural networks with many more parameters than training data},
  author={Dziugaite, Gintare Karolina and Roy, Daniel M},
  journal={arXiv preprint arXiv:1703.11008},
  year={2017}
}

@article{petzka2021relative,
  title={Relative flatness and generalization},
  author={Petzka, Henning and Kamp, Michael and Adilova, Linara and Sminchisescu, Cristian and Boley, Mario},
  journal={Advances in neural information processing systems},
  volume={34},
  pages={18420--18432},
  year={2021}
}

@book{krantz2012realanalytic,
    author ={Krantz, Steven G. and Parks, Harold R.},
    title = {A primer of real analytics functions},
    publisher = {Birkhäuser Boston, MA},
    year = {2012}
}

@article{whitney1965varieties,
    author = {Whitney, H.},
    title = {Local properties of analytic varieties},
    journal = {Differential and Combinatorial Topology, Princeton Univ~Press},
    pages = {205--244},
    year = {1965}
}

@inproceedings{jiangfantastic,
  title={Fantastic Generalization Measures and Where to Find Them},
  author={Jiang, Yiding and Neyshabur, Behnam and Mobahi, Hossein and Krishnan, Dilip and Bengio, Samy},
  booktitle={International Conference on Learning Representations},
  year={2019}
}

@inproceedings{liang2019fisher,
  title={Fisher-rao metric, geometry, and complexity of neural networks},
  author={Liang, Tengyuan and Poggio, Tomaso and Rakhlin, Alexander and Stokes, James},
  booktitle={The 22nd international conference on artificial intelligence and statistics},
  pages={888--896},
  year={2019},
  organization={PMLR}
}

@inproceedings{keskar2016large,
  title={On Large-Batch Training for Deep Learning: Generalization Gap and Sharp Minima},
  author={Keskar, Nitish Shirish and Mudigere, Dheevatsa and Nocedal, Jorge and Smelyanskiy, Mikhail and Tang, Ping Tak Peter},
  booktitle={International Conference on Learning Representations},
  year={2016}
}

@inproceedings{kolchinskycaveats,
  title={Caveats for information bottleneck in deterministic scenarios},
  author={Kolchinsky, Artemy and Tracey, Brendan D and Van Kuyk, Steven},
  booktitle={International Conference on Learning Representations},
  year={2019}
}

@inproceedings{ting2025flatness,
    author = {Han, Ting and Adilova, Linara and Petzka, Henning and Kleesiek, Jens and Kamp, Michael},
    booktitle = {Advances in Neural Information Processing Systems},
    pages = {},
    title = {Flatness is Necessary, Neural Collapse is Not: Rethinking Generalization via Grokking},
    year = {2025}
}

@article{srivastava2014dropout,
  title={Dropout: a simple way to prevent neural networks from overfitting},
  author={Srivastava, Nitish and Hinton, Geoffrey and Krizhevsky, Alex and Sutskever, Ilya and Salakhutdinov, Ruslan},
  journal={The journal of machine learning research},
  volume={15},
  number={1},
  pages={1929--1958},
  year={2014},
  publisher={JMLR. org}
}

@misc{bhargava2021generalization,
      title={Generalization in NLI: Ways (Not) To Go Beyond Simple Heuristics}, 
      author={Prajjwal Bhargava and Aleksandr Drozd and Anna Rogers},
      year={2021},
      eprint={2110.01518},
      archivePrefix={arXiv},
      primaryClass={cs.CL}
}

@article{DBLP:journals/corr/abs-1908-08962,
  author    = {Iulia Turc and
               Ming{-}Wei Chang and
               Kenton Lee and
               Kristina Toutanova},
  title     = {Well-Read Students Learn Better: The Impact of Student Initialization
               on Knowledge Distillation},
  journal   = {CoRR},
  volume    = {abs/1908.08962},
  year      = {2019},
  url       = {http://arxiv.org/abs/1908.08962},
  eprinttype = {arXiv},
  eprint    = {1908.08962},
  bibsource = {dblp computer science bibliography, https://dblp.org}
}

@article{zhang2015character,
  title={Character-level convolutional networks for text classification},
  author={Zhang, Xiang and Zhao, Junbo and LeCun, Yann},
  journal={Advances in neural information processing systems},
  volume={28},
  year={2015}
}

@inproceedings{cheng2020club,
  title={CLUB: A Contrastive Log-ratio Upper Bound of Mutual Information},
  author={Cheng, Yu and Song, Jiaming and Miao, Yunchen},
  booktitle={International Conference on Machine Learning},
  year={2020},
  organization={PMLR}
}

@article{oord2018representation,
  title={Representation Learning with Contrastive Predictive Coding},
  author={van den Oord, Aaron and Li, Yazhe and Vinyals, Oriol},
  journal={arXiv preprint arXiv:1807.03748},
  year={2018}
}

@inproceedings{belghazi2018mine,
  title={Mutual Information Neural Estimation},
  author={Belghazi, Mohamed Ishmael and Baratin, Aristide and Rajeshwar, Devansh and Ozair, Sherjil and Bengio, Yoshua and Courville, Aaron and Hjelm, R Devon},
  booktitle={International Conference on Learning Representations},
  year={2018}
}

@book{reichenbach1991direction,
  title={The direction of time},
  author={Reichenbach, Hans},
  volume={65},
  year={1991},
  publisher={Univ of California Press}
}

@inproceedings{garrido2023rankme,
  title={Rankme: Assessing the downstream performance of pretrained self-supervised representations by their rank},
  author={Garrido, Quentin and Balestriero, Randall and Najman, Laurent and Lecun, Yann},
  booktitle={International conference on machine learning},
  pages={10929--10974},
  year={2023},
  organization={PMLR}
}

@inproceedings{patel2026learning,
  title={Learning Compact Representations via Intrinsic Dimension Regularization},
  author={Patel, Laksh and Bukkapatnam, Kaustubh S and Batra, Soham},
  booktitle={ICLR 2026 Workshop on Geometry-grounded Representation Learning and Generative Modeling},
  year={2026}
}

@article{agrawal2022alpha,
  title={Assessing Representation Quality in Self-Supervised Learning by measuring eigenspectrum decay},
  author={Agrawal, Kumar K and Mondal, Arnab Kumar and Ghosh, Arna and Richards, Blake},
  journal={Advances in Neural Information Processing Systems},
  volume={35},
  pages={17626--17638},
  year={2022}
}

\appendix
\appendix
\onecolumn
\newcommand{\RR}{\mathbb{R} }
\newcommand{\NN}{\mathbb{N}}
\newcommand{\Exp}[2]{\mathop{{}\mathbb{E}_{#1}} \left [ #2 \right ] }

\section{Proof of Theorem~\ref{thm:finiteMI}}
\label{app:proof}

To prove the theorem, we use the following proposition, which extends Proposition~3.5 of~\cite{adilovainformation} to non-constant real analytic activation functions. The proof then directly follows from Theorem~3.3 in~\cite{adilovainformation}.

\newcommand{\set}[1]{\mathcal{#1}}
\begin{proposition}\label{prop:finiteExp}
  Consider a non-constant deterministic DNN function $f{:}\ \set{X}\to\mathbb{R}$ constructed with finitely many layers, a finite number of neurons per layer, and real analytic activation functions. Let $X$ be a continuously distributed RV with PDF $p_x$ that is bounded ($0\le p(x) <C_p$ for all $x\in\set{X}$) and has support in a compact and connected set $K$. Then the conditional expectation $\mathbb{E}[\log(|f(X)|)]$ is finite.
\end{proposition}

\begin{proof}
By the law of unconscious statistician, we have
\begin{equation}\label{finiteExpectation} 
\Exp{X}{\log(|f(X)|)} = \int_K \log(|f(x)|) p(x) \mathrm{d}x 
\end{equation}
As $p(x)$ is bounded and $K$ is bounded, it is clear that  $\Exp{X}{\log(|f(X)|)} < \infty$. It remains to investigate the behavior around $f(X)=0$ to show $\Exp{X}{\log(|f(X)|)} > -\infty$. Since $f$ is a composition of real analytic functions, $f$ is itself real analytic. We will make use of the well-behaved structure of real analytic functions to bound the expectation.

The idea of the proof is as follows: The set of zeros of a real analytic function is a set of measure zero. However, this does not suffice to guarantee finiteness of $\Exp{X}{\log(|f(X)|)}$ as the behavior around these zeros matters. In one dimension, we can directly use the fact that an analytic function $f:\RR \rightarrow \RR$ locally behaves polynomially around zeros. In higher dimensions, an analytic function can remain 0 along submanifolds of lower dimension. The natural analog to the behavior in one dimension is that the function behaves polynomially as a function of the distance to the set of zeros. Once we know that $f$ behaves polynomially, we can use that $\int_0^\epsilon \log(r) q(r) \mathrm{d}r$ is finite for any polynomial $q(r)$. 

Technically, we apply the Lojasiewicz's Structure Theorem (see Theorem~6.3.3 in \cite{krantz2012realanalytic}), which states that the set of zeros of a non-constant real analytic function is a finite collection of algebraic varieties of dimension at most $n-1$. In particular, the set of zeros is a null set and eq.~\eqref{finiteExpectation} is well-defined. Let $\mathcal{Z}_{f,i}, i=1,\ldots,k$ denote the connected components of $\mathcal{Z}_{f}$. 



Moreover, the Lojasiewicz inequality (see Theorem~6.3.4 in \cite{krantz2012realanalytic}), gives that around zeros of $f$, the function $f$  behaves polynomially in the distance to the zero: That is, let $\mathcal{Z}_{f}$ be the set of zeros of $f$ and $dist(x,\mathcal{Z}_{f})$ the distance function from $x$ to $\mathcal{Z}_{f}$. Let $U$ be an open neighborhood with non-zero intersection with $\mathcal{Z}_{f}$. Then, for each compact subset $E$ of $U$, there is a constant $C_L>0$ and a natural number $q$ such that
\begin{equation}\label{lojaInequality}
|f(x)| \geq C_L\cdot  dist(x,\mathcal{Z}_{f})^q
\end{equation}
for all $z$ in $E$.

Consider now $\epsilon$-neighborhoods $N_i^\epsilon, i=1,\ldots k$ around the connected components $\mathcal{Z}_{f,i}$, where $\epsilon$ is chosen such that $|f(x)|<1$ for all $x\in N_i^\epsilon \cap K$ and all $i$. Then,
\begin{equation}
\Exp{X}{\log(|f(X)|)} > \sum_i^{finite}  \int_{N_i^\epsilon} \log(|f(x)|)\cdot p(x) \ \mathrm{d}x  +  \int_{K\setminus \bigcup_i N_i^\epsilon} \log(|f(x)|)\cdot p(x)\ \mathrm{d}x.
\end{equation}
The last integral is bounded below by boundedness of $K$, continuity of $f$, and the fact that $x$ has at least distance $\epsilon>0$  to any zero of $f$. It remains to show that each  $\int_{N_i^\epsilon} \log(|f(x)|)\cdot  p(x)\ \mathrm{d}x>-\infty$. We fix $i$ in the following and remove it from the notation for readability.

We use the Lojasiewicz inequality \eqref{lojaInequality} to find $C_L>0$ and $q\in \NN$ such that $|f(x)| \geq C_L\cdot  dist(x,\mathcal{Z}_{f})^q$ on $N^\epsilon$. 

By using Whitney stratification \cite{whitney1965varieties}, a real analytic algebraic variety can be decomposed into a finite number of connected smooth submanifolds, $\mathcal{Z}_{f} = \bigcup_j^{finite} M_{j}$. Let $d_{j}$ denote the dimension of $M_{j}$. We know that $0\leq d_{j} < n$.

Using the decomposition, we will need that $f$ locally behaves polynomially in distance to each $M_j$. Since eq.~\eqref{lojaInequality} only measures the distance to the variety $\mathcal{Z}_{f}$, we will need to partition the space into regions where the distance to $\mathcal{Z}_{f}$ is defined by the distance to $M_j$. For this, we let  
\[N_{j}^{\epsilon} = \left \{ x \in N^{\epsilon}\ |\ dist(x,M_j) = dist(x,\mathcal{Z}_{f}) \right \}\ .\]
Then 
\[\bigcup_j N_{j}^{\epsilon} = N^\epsilon \]
since every distance needs to be realized by at least one of the connected smooth manifold. 

For each $j$, we introduce local coordinate systems $x^{j}= (x^{j}_1,\ldots, x^{j}_{d_j})$ of the manifold $M_{j}$ and $x^{j,\perp} = (x^{j,\perp}_1,\ldots,x^{j,\perp}_{n-d_j})$ of its normal space. 
Now we calculate the integral:

\begin{align*}
\int_{N^\epsilon} \log(|f(x)|) p(x) dx &\geq \sum_j^{finite} \int_{N_{j}^{\epsilon}}  \log(|f(x)|)\cdot p(x) \mathrm{d}x\\
&\geq \sum_j^{finite} \int_{N_{j}^{\epsilon}}  \log(|f(x)|)\cdot C_p\ \mathrm{d}x\\ 
&\geq \sum_j^{finite} \int_{N_{j}^{\epsilon}}  \log(C_L\cdot  dist(x,Z)^q)\cdot C_p\ \mathrm{d}x\\
 & = \sum_j^{finite} \int_{N_{j}^{\epsilon}} \log(C_L\cdot  dist(x,M_{j})^q) \cdot C_p\ \mathrm{d}x\\
  & \geq  \sum_j^{finite} \int_{M_j}\int_{B^{n-d}_\epsilon(0)} \log(C_L\cdot  dist(x,M_j)^q) \cdot C_p\ \mathrm{d}x^\perp\  \mathrm{d}x^j\\
 & =  \sum_j^{finite} \int_{M_j} vol_{n-d_j-1}(S^{n-d_j-1}) \int_{r=0}^\epsilon \log(C_L\cdot r^q) \cdot C_p\ r^{n-d_j-1}\ \mathrm{d}r\ \mathrm{d}x^j\\
 &  =  \sum_j^{finite} C_p\cdot vol_{d_j}(M_j) \cdot vol_{n-d_j-1}(S^{n-d_j-1})\\
 &\hspace{2cm}\cdot \left (\frac{\epsilon^{n-d_j}}{n-d_j} \log(C_L) + q\cdot \int_{r=0}^\epsilon \log(r) \cdot r^{n-d_j-1} \mathrm{d}r \right )\\
 & > -\infty .\\
\end{align*}
The final inequality holds since $M_j$ lies in a bounded set and the last integral is finite for $n-d_j-1\geq0$, i.e., $d_j< n$.
\end{proof}

\section{Experiment Details}
\label{app:experiment}

\subsection{Correlation Estimation}
\label{sec:correlation_computation}
Mathematically, if $\{a_i^j\}$ and $\{b_i^j\}$, $i=1,\dots, N_j$, are quantities recorded in the $j$-th experiment by varying a certain parameter, let $r_{a,i}^j$ and $r_{a,i}^j$ denote the ranks of $a_i^j$ and $b_i^j$ such that $r_{a,i}^j=1$ if $a_i^j>a_\ell^j$ for all $\ell\neq i$, etc.
The ranks of the $j$-th experiment are then min-max normalized to the range of $[0,1]$, and the normalized ranks are collected in a vector $\tilde{r}_a=[\tilde{r}_{a,1}^1,\tilde{r}_{a,2}^1,\dots,\tilde{r}_{a,N_1}^1,\tilde{r}_{a,1}^2,\dots,\tilde{r}_{a,N_M}^M]$.
We then report the Pearson correlation coefficients between $\tilde{r}_a$ and $\tilde{r}_b$.

\subsection{Training Hyper Parameters}
\label{sec:hyperparameters}
For all the setups in CEB experiments we employ Adam optmizer with learning rate $1e-3$.
We trained models for $150$ epochs with exponential learning rate scheduler with $\gamma=0.97$.
Batch size is selected as following:
\begin{enumerate}
    \item FC + FMNIST: $64$
    \item LeNet5 + FMNIST: $64$
    \item WRN28-4 + CIFAR-10: $256$
    \item DenseNet-121 + CIFAR-100: $256$
\end{enumerate}

For the setups in Gaussian dropout experiments the following hyper parameters were used:
\begin{enumerate}
    \item ResNet18 + CIFAR-10: batch size $128$, learning rate $0.1$, training for $100$ epochs with SGD with weight decay of $5e-4$, momentum $0.9$ and cosine annealing learning rate scheduler.
    \item VGG11 + SVHN: batch size $256$, learning rate $0.01$, training for $150$ epochs with SGD with weight decay of $5e-4$, momentum $0.9$ and cosine annealing learning rate scheduler.
    \item DenseNet-121 + CIFAR-100: batch size $256$, learning rate $0.1$, training for $200$ epochs with SGD with weight decay of $5e-4$, momentum $0.9$ and cosine annealing learning rate scheduler.
    \item mini-BERT + AG News: batch size $256$, learning rate $1e-5$, training for $20$ epochs with AdamW with weight decay of $1e-2$ and cosine annealing learning rate scheduler.
\end{enumerate}

\subsection{Per-seed Confidence Intervals for Dropout Setup}
\label{sec:per_seed_confidence}
In Table~\ref{app:tab:seed_correlations}, we report per-seed correlations for the Gaussian dropout setup, aggregated to obtain a mean and standard deviation, alongside the overall (pooled) correlation for comparison.
The mean values closely match those obtained from the pooled correlations, while the confidence intervals for strong correlations are notably narrow.
Note that these correlations differ slightly from those in the main paper, since one of the setups was run with a different set of random seeds and could therefore not be included in the per-seed aggregation.

\definecolor{highlightrow}{gray}{0.9}

\begin{table}[!ht]
\centering

\begin{subtable}{0.5\textwidth}
    \centering
    \begin{tabular}{l S[table-format=-1.2] S[table-format=-1.2]}
        \toprule
        & {Train} & {Test} \\
        \midrule
        acc $\|$ $\lambda$  & -0.06 & -0.28 \\
        gen $\|$ $\lambda$  &  0.05 &  {---} \\
        MI  $\|$ $\lambda$  & -0.53 & -0.3 \\
        NC  $\|$ $\lambda$  &  0.96 &  0.94 \\
        \cmidrule(lr){1-3}
        \rowcolor{highlightrow}
        \textbf{MI $\|$ NC} & \textbf{-0.52} & \textbf{-0.28} \\
        \cmidrule(lr){1-3}
        acc $\|$ MI         & -0.19 &  0.08 \\
        acc $\|$ NC         & -0.09 & -0.32 \\
        gen $\|$ MI         &  0.11 &  {---} \\
        gen $\|$ NC         &  0.04 &  {---} \\
        \bottomrule
    \end{tabular}
    \caption{Gaussian dropout ($\lambda$ sweep).}
    \label{app:tab:gaus}
\end{subtable}%
\hfill
\begin{subtable}{0.5\textwidth}
    \centering
    \begin{tabular}{l *{2}{>{\centering\arraybackslash}p{1.7cm}}}
        \toprule
        & {Train} & {Test} \\
        \midrule
        acc $\|$ $\lambda$  & $-0.08\pm0.08$ & $-0.27\pm0.14$ \\
        gen $\|$ $\lambda$  &  $0.04\pm0.09$ &  {---} \\
        MI  $\|$ $\lambda$  & $-0.53\pm0.04$ & $-0.3\pm0.14$ \\
        NC  $\|$ $\lambda$  &  $0.96\pm0.01$ &  $0.94\pm0.02$ \\
        \cmidrule(lr){1-3}
        \rowcolor{highlightrow}
        \textbf{MI $\|$ NC} & $-0.51\pm0.04$ & $-0.28\pm0.14$ \\
        \cmidrule(lr){1-3}
        acc $\|$ MI         & $-0.14\pm0.04$ &  $0.06\pm0.11$ \\
        acc $\|$ NC         & $-0.09\pm0.08$ & $-0.27\pm0.14$ \\
        gen $\|$ MI         &  $0.13\pm0.11$ &  {---} \\
        gen $\|$ NC         &  $0.02\pm0.08$ &  {---} \\
        \bottomrule
    \end{tabular}
    \caption{Gaussian dropout per-seed confidence.}
    \label{app:tab:gaus_seed}
\end{subtable}
\caption{Rank correlations between MI, NC, accuracy (\emph{acc}), and generalization gap (\emph{gen}) across Gaussian dropout experiments. Left table reports combined correlation, while right aggregates per-seed and reports confidence intervals.}
\label{app:tab:seed_correlations}
\end{table}

\subsection{Why DoE?}
In this section we summarize the theoretical basis of DoE, explain why it yields a lower bound on MI, and describe the neural density models used to approximate the marginal and conditional distributions.
We also include the role of the logistic-based DoE formulation and the construction of negative samples.

For any random variable $Z$ with true density $p(z)$, the entropy is
\begin{equation*}
    H(Z) = -\mathbb{E}_{Z}[\log p(z)].
\end{equation*}
DoE replaces the intractable $\log p(z)$ with the log-density of a neural estimator $q_{\phi}(z)$, trained by maximum likelihood on samples of $z$:
\begin{equation*}
    \hat{H}(Z) = -\mathbb{E}_{Z}[\log q_{\phi}(z)].
\end{equation*}
Using the standard decomposition,
\begin{equation*}
    -\mathbb{E}_{Z}[\log q_{\phi}(z)]
    =
    -\mathbb{E}_{p(z)}[\log p(z)] + \mathrm{KL}\bigl(p(z)\,\|\,q_{\phi}(z)\bigr),
\end{equation*}
we immediately obtain
\begin{equation*}
    \hat{H}(Z) = H(Z) + \mathrm{KL}(p(z)\,\|\,q_{\phi}(z)) \ge H(Z).
\end{equation*}
Thus, a variational estimator cannot underestimate entropy; it always overestimates it by the non-negative KL divergence.

The same argument applies to the conditional entropy:
\begin{equation*}
    \widehat{H}(Z|X)
    = -\mathbb{E}_{p(x,z)}[\log q_{\theta}(z|x)]
    = H(Z|X) + \mathbb{E}_{p(x)}\mathrm{KL}(p(z|x)\,\|\,q_\theta(z|x)).
\end{equation*}
Again, the estimate is an overestimate of the true conditional entropy.

MI can be expressed with entropies as follows
\begin{equation*}
    I(X;Z) = H(Z) - H(Z|X).
\end{equation*}
DoE substitutes the variational surrogates:
\begin{equation*}
    \hat{I}_{\text{DoE}} = \hat{H}(Z) - \widehat{H}(Z|X).
\end{equation*}
Substituting their decompositions yields
\begin{align*}
\hat{I}_{\text{DoE}}
&= H(Z) + \mathrm{KL}(p(z)\|q_\phi) - \left( H(Z|X) + \mathbb{E}_{p(x)} \mathrm{KL}(p(z|x)\|q_\theta) \right) \\
&= I(X;Z) - \Big( \mathbb{E}_{p(x)}\mathrm{KL}(p(z|x)\|q_\theta) - \mathrm{KL}(p(z)\|q_\phi) \Big).
\end{align*}
The conditional KL is typically much larger than the marginal KL.
No matter that \(p(z| x)\) is usually a unimodal distribution, while \(p(z)\) is a
multimodal mixture over all inputs, estimating
\(H(Z\mid X)\) often requires a more expressive model because learning
the mapping \((x,z)\mapsto \text{conditional code length}\) is
estimator–wise more difficult than modeling the marginal \(p(z)\).
Hence,
\begin{equation*}
    \mathbb{E}_{p(x)}\mathrm{KL}(p(z|x)\|q_\theta) \gg \mathrm{KL}(p(z)\|q_\phi),
\end{equation*}
which implies
\begin{equation*}
    \hat{I}_{\text{DoE}} \le I(X;Z).
\end{equation*}
Thus DoE produces a lower bound on MI through upper bounds on corresponding entropies.

\paragraph{Neural Density Models Used in DoE}
DoE models both the marginal and conditional densities with neural networks:
\begin{itemize}
    \item A marginal density estimator $q_{\phi}(z)$, trained on samples $\{z_i\}$.
    \item A conditional density estimator $q_{\theta}(z|x)$, trained on pairs $\{(x_i, z_i)\}$.
\end{itemize}
The density estimator employed in the logistic-based DoE variant is a factorized logistic mixture model.
The logistic parameterization is more flexible than a Gaussian model and avoids the assumption of unimodality, which is particularly important for $p(z|x)$.

For the ``logistic'' option, DoE models the marginal or conditional distribution of $Y \in \mathbb{R}^d$ as a diagonal Logistic distribution with parameters $\mu \in \mathbb{R}^d, s = \exp(\ln\mathrm{var}) \in \mathbb{R}^d_{>0}$.
For a single dimension, the Logistic density is
\[
p(y \mid \mu, s)
    = \frac{\exp\!\left( -\frac{y - \mu}{s} \right)}%
           {s\left(1 + \exp\!\left( -\frac{y - \mu}{s} \right)\right)^{2}}.
\]
Define the standardized variable $w = \frac{y - \mu}{s}$.
The exact negative log-likelihood is
\[
-\log p(y \mid \mu, s) = w + 2\log\!\bigl(1 + e^{-w}\bigr) + \log s.
\]
In the vector case ($y,\mu,s \in \mathbb{R}^d$), the code computes
\[
\mathcal{L}(Y)
  = \frac{1}{B}
    \sum_{b=1}^B
    \sum_{i=1}^d 
    \left[
         w_{b,i}
        + 2 \log\!\bigl(1 + e^{-w_{b,i}}\bigr)
        + \log s_i
    \right],
\]
where $B$ is the batch size.
In the implementation we compute:
\[
\texttt{whitened} = (Y - \mu)/s \quad \longleftrightarrow \quad w,
\]
\[
\texttt{adjust} = \mathrm{softplus}(-w)
    = \log\left( 1 + e^{-w} \right),
\]
and the returned value
\[
\texttt{negative\_ln\_prob}
  = \mathrm{mean}_b\!\left[
          \sum_{i=1}^d 
          \Bigl( w_{b,i}
            + 2\,\mathrm{softplus}(-w_{b,i})
            + \ln s_i
          \Bigr)
    \right]
\]
is exactly the multidimensional negative log-likelihood of a factorized Logistic distribution.

\subsubsection{MI Estimators Comparison}
\label{sssec:mi_compare}
We estimated MI using MINE~\cite{belghazi2018mine} and CLUB~\cite{cheng2020club} in one of the configurations, ResNet18 on CIFAR-10.
With CLUB we obtained quite good convergence (Fig.~\ref{fig:other_mi_converg}(b)) and a qualitatively similar picture to DoE except for two outliers (Fig.~\ref{fig:other_mi_correl}(b, c)).
It is visible still, that DoE achieves significantly better convergence (Fig.~\ref{fig:doe}(a)).
\begin{figure}[!ht]
    \centering
         \begin{subfigure}[b]{0.3\textwidth}
         \centering
         \includegraphics[width=\linewidth]{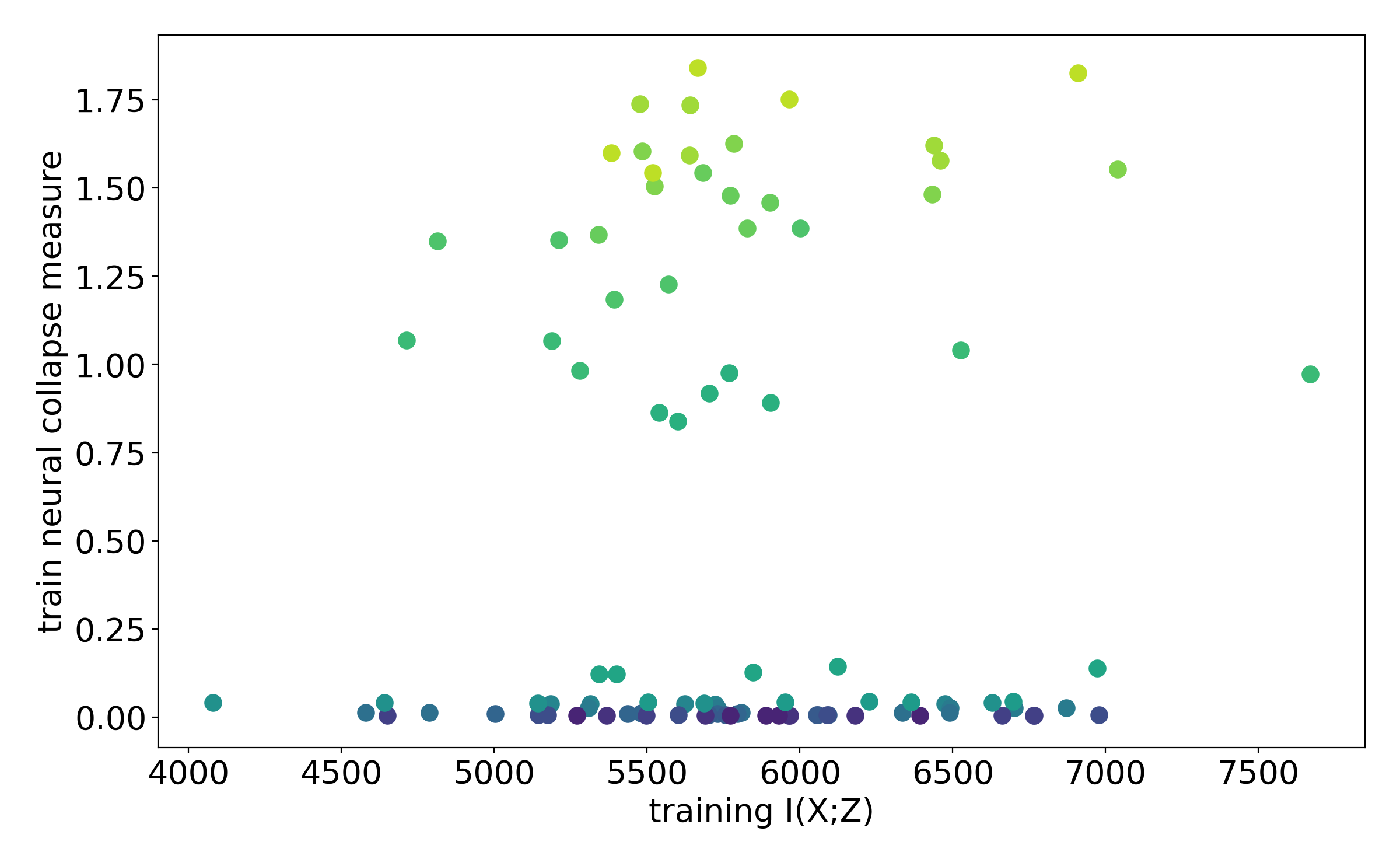}
         \caption{MI estimated with  MINE}
     \end{subfigure}
     \hfill
              \begin{subfigure}[b]{0.3\textwidth}
         \centering
         \includegraphics[width=\linewidth]{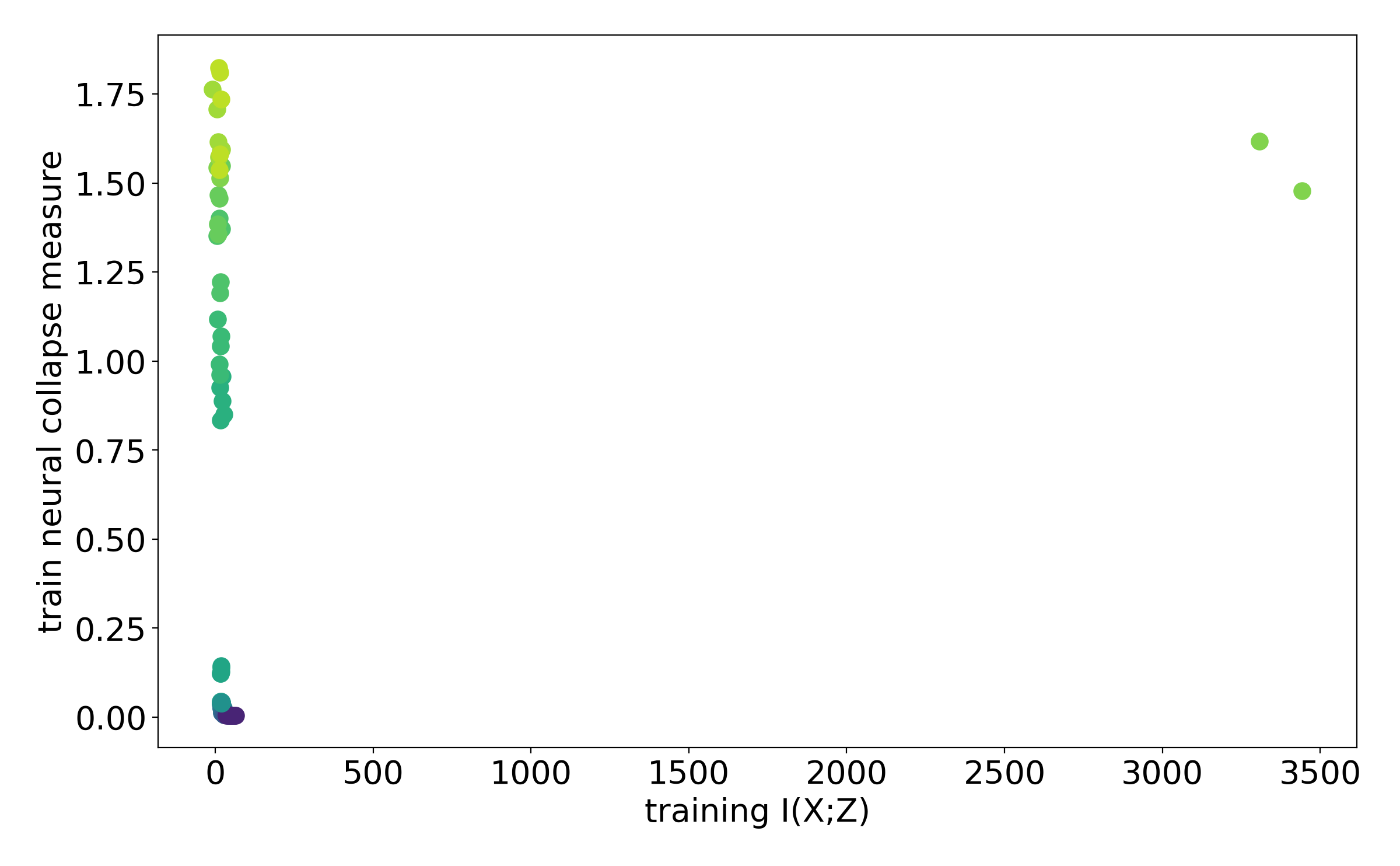}
         \caption{MI estimated with CLUB}
     \end{subfigure}
     \hfill
              \begin{subfigure}[b]{0.3\textwidth}
         \centering
         \includegraphics[width=\linewidth]{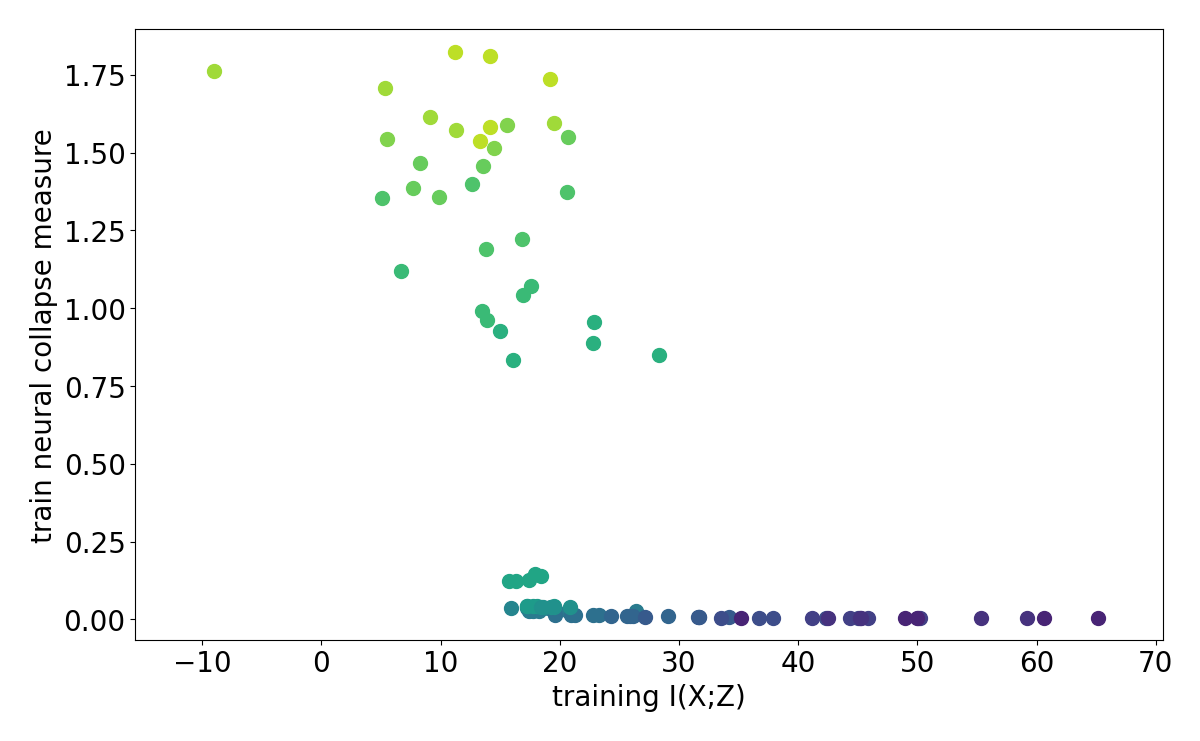}
         \caption{MI with CLUB, w/o outliers}
     \end{subfigure}
    \caption{Correlation plots between MI and geometric compression produced by MINE and CLUB estimator on ResNet18 representations with CIFAR-10 data.}
     \label{fig:other_mi_correl}
\end{figure}
With MINE the picture is more distorted (Fig.~\ref{fig:other_mi_correl}(a)), but it has nearly no convergence in the same setup as DoE and CLUB (Fig.~\ref{fig:other_mi_converg}(a)), since as explained in \cite{mcallester2020formal} MINE requires an exponential amount of samples for convergence.
\begin{figure}[!ht]
    \centering
         \begin{subfigure}[b]{0.45\textwidth}
         \centering
         \includegraphics[width=\linewidth]{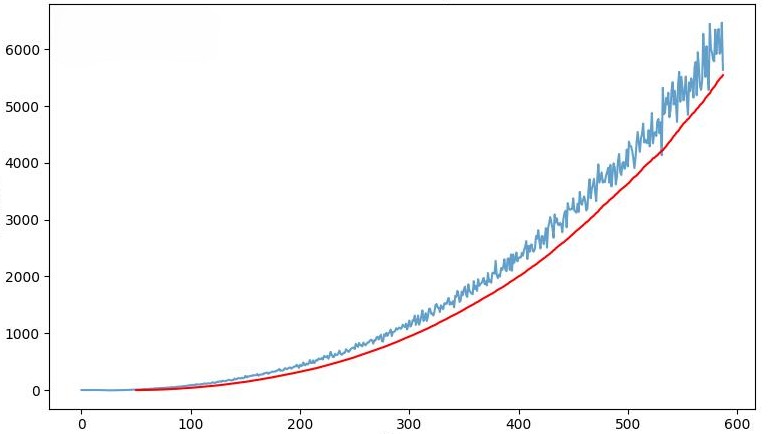}
         \caption{Train data convergence of MINE estimator}
     \end{subfigure}
     \hfill
              \begin{subfigure}[b]{0.45\textwidth}
         \centering
         \includegraphics[width=\linewidth]{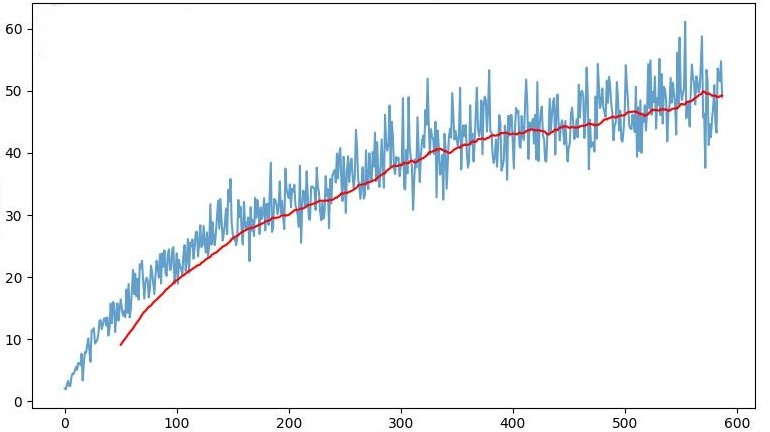}
         \caption{Train data convergence of CLUB estimator}
     \end{subfigure}
    \caption{MINE and CLUB estimators convergence, when trained on ResNet18 representations with CIFAR-10 data.}
     \label{fig:other_mi_converg}
\end{figure}
We did not consider InfoNCE~\cite{oord2018representation} since its estimation is capped by the batch size used for estimation, and our measured MI can be significantly larger than a batch size with which we can train the estimator.

\subsection{DoE Hyperparameters}
\label{sec:doe_hyperparam}
In our implementation the encoder $f_\psi(x)$ produces deterministic features $z = f_\psi(x)$, but we use dropout inside the encoder to obtain multiple stochastic views of the same input.
For each input $x_i$ we generate
\begin{equation}
    z_i^{(1)}, z_i^{(2)}, \ldots, z_i^{(K)},
\end{equation}
where each $z_i^{(k)} = f_\psi(x_i;\,\epsilon_k)$ corresponds to an independent dropout mask $\epsilon_k$.
These samples should not be treated as independent draws from the marginal $p(z)$; they share the same underlying input 
and differ only due to injected noise.
Consequently, they cannot be used as negative samples against each other.

If the dropout-generated samples $\{z_i^{(k)}\}$ were used as negatives for the same $x_i$, the objective would force the conditional density estimator $q_\theta(z|x)$ to assign low likelihood to dropout variants of the same input.
This contradicts the intended semantics of $q_\theta(z|x)$, which should assign high density to all stochastic realizations of $z$ arising from the same input $x$.
In other words, these samples approximate samples from the same conditional distribution $p(z|x_i)$ and are therefore all positive examples, not negatives.

For each input $x_i$ we use:
\begin{itemize}
    \item \textbf{Positive set:} all dropout-generated features 
        $\{z_i^{(1)},\ldots,z_i^{(K)}\}$.
    \item \textbf{Negative set:} dropout samples of \emph{other} inputs,
        that is
        \begin{equation}
            \mathcal{N}_i
            =
            \{ z_j^{(k)} \;|\; j\neq i,\; k=1,\ldots,K\}.
        \end{equation}
        These approximate draws from the empirical marginal $p(z)$ and are valid negatives for $x_i$.
\end{itemize}
Thus, for a minibatch of size $B$ with $K$ dropout samples each, we obtain:
\begin{itemize}
    \item $K$ positive samples for each of the $B$ inputs;
    \item $(B-1)K$ negative samples per input.
\end{itemize}

The logistic DoE objective becomes
\begin{equation}
\mathcal{L}_{\text{cond}}
=
\sum_{i=1}^{B}
\sum_{k=1}^{K}
\left[
    \log q_{\theta}(z_i^{(k)}|x_i)
    -
    \log\left(
        \sum_{z \in \mathcal{N}_i \cup \{z_i^{(1)},\ldots,z_i^{(K)}\}}
        q_{\theta}(z|x_i)
    \right)
\right].
\end{equation}
Note that all dropout samples of the same input appear in both the numerator (as positives) and in the normalizer of the denominator (to retain the standard noise-contrastive form), but they are never treated as negatives.

This construction preserves the semantics of the conditional model:
\begin{itemize}
    \item $q_\theta(z|x_i)$ learns to assign high likelihood to all natural 
          stochastic variations of features arising from $x_i$,
    \item and low likelihood to samples originating from other inputs $x_j$.
\end{itemize}
Thus the use of dropout increases the effective sample size of $p(z|x)$ without corrupting the negative sampling structure, leading to a more faithful estimation of the conditional entropy.

DoE estimator of conditional entropy was selected depending on the data: For images we employ a simple convolutional architecture and for text a two-layer attention network.
We train both estimators with AdamW optimizer with learning rate $1e-4$, batch size $256*4$ (with $4$ samples from dropout) and gradient clipping to $1$.
Since convergence of MI estimators is an important characteristic of the quality of the obtained values, we demonstrate here an example of one of the setups measurements.
The training data allows for a converged estimation, while test data is close to the convergence, but not very precise.
Unfortunately, estimation on one class, needed for conditional MI, cannot converge properly since amount of samples is too small.
\begin{figure}[!ht]
    \centering
         \begin{subfigure}[b]{0.3\textwidth}
         \centering
         \includegraphics[width=\linewidth]{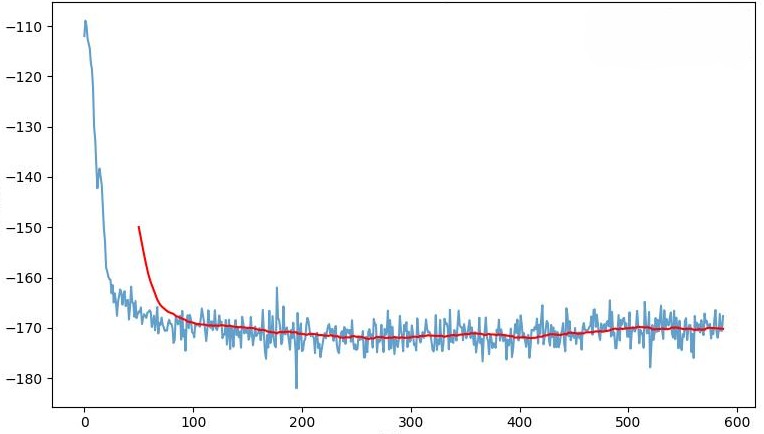}
         \caption{Train data}
     \end{subfigure}
     \hfill
              \begin{subfigure}[b]{0.3\textwidth}
         \centering
         \includegraphics[width=\linewidth]{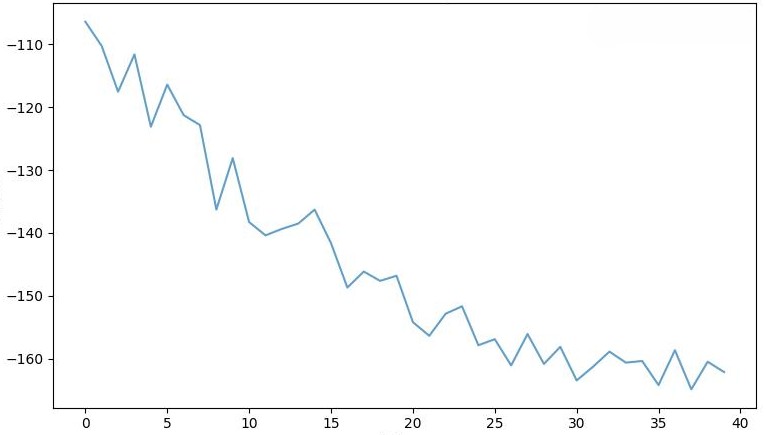}
         \caption{Test data}
     \end{subfigure}
     \hfill
              \begin{subfigure}[b]{0.3\textwidth}
         \centering
         \includegraphics[width=\linewidth]{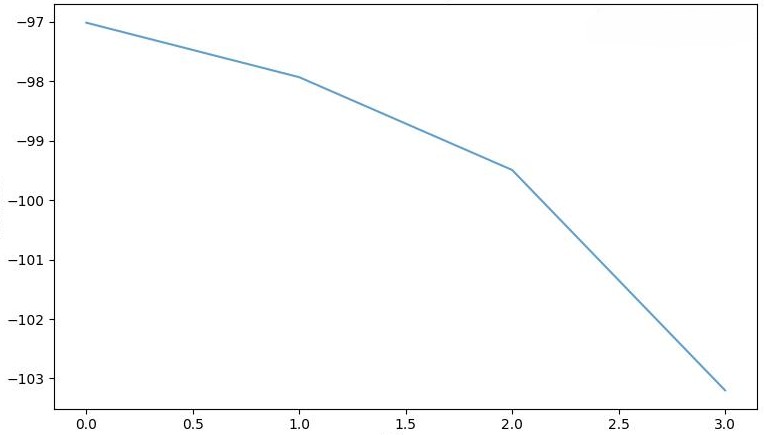}
         \caption{One class train data}
     \end{subfigure}
    \caption{Convergence of the estimator of $H(Z|X)$ for DoE on ResNet18 representations with CIFAR-10 data.}
     \label{fig:doe}
\end{figure}

\section{Theoretical Argument for Negative Correlation}
\label{sec:app:numerical_example}
We now provide some theoretical arguments that should intuitively explain the observed negative correlation between geometric and information-theoretic compression. To this end, we consider a simplified setting. 

We consider a set of points from $\mathcal{X}$ with corresponding classes from $\mathcal{Y}$.
We define a stochastic encoder, which maps inputs $X \in \mathcal{X}$ to a two-dimensional latent space $\mathcal{Z}$, encoding a datapoint $x$ with class $y$ to a random variable $Z\sim\mathcal{N}(\mu(x),\sigma^2(x))$, where:
\begin{enumerate}
    \item $\mu(x)$: The encoder is assumed to distribute all samples of a given class around a learned class centroid. While we assume that this mapping is bijective, we emulate the spread of the encoder output due to differences in the input via drawing from a Gaussian distribution. More specifically, we define the function $\mu$ via drawing, for each datapoint $(x,y)$, a sample from $\mathcal{N}(\mu_y,\sigma_z^2 I)$. Here, $\mu_y$ corresponds to the class centroids, while $\sigma_z^2$ corresponds to the spread of the encoder function $\mu$ (which we assume identical for every class, for the sake of simplicity).    
    \item $\sigma^2(x)=\sigma_n^2$: The encoder adds noise with a fixed variance to the latent representation. 
\end{enumerate}
    
Note that since the mapping from $X$ to $\mu(X)$ is bijective, its mutual information is infinite, also $I(X;Z)=I(\mu(X);Z)$.
Using the formula for the MI of jointly Gaussian random variables, we obtain
\begin{equation}
    I(X;Z|Y) = \log\left(1+\frac{\sigma_z^2}{\sigma_n^2}\right).
\end{equation}
Hence, large values of $\sigma_n^2$ lead to small values of $I(X;Z)$. For the clustering perspective, small values of $\sigma_z^2$ and $\sigma_n^2$ lead to latent representations that are strongly clustered in a geometric sense.

We simulate this system by drawing the class means $\mu_y$ from $\mathcal{N}(0,I)$ and varying the parameters $\sigma_z^2$ and $\sigma_n^2$ to study their influence on $NC$. To compensate for random effects, we repeat this procedure 50 times.

The resulting measurements are shown in Fig.~\ref{fig:theoretical}. They demonstrate that information-theoretic compression, i.e., lowering of $I(\mu(X), Z)$, can result from two different causes: 
\begin{enumerate}
    \item A large encoder noise variance $\sigma_n^2$ (lower left corner of Fig.~\ref{fig:theoretical}(a)), which indicates that noise causes class-specific distributions to overlap;
    \item A small variance $\sigma_z^2$ of the encoder mean (upper right corner of Fig.~\ref{fig:theoretical}(a)), which indicates that samples from the same class are mapped closely in latent space, at least via the deterministic part of $\mu(x)$ of the encoder.
\end{enumerate}
    
In contrast, the neural collapse measure $NC$ is small only if both $\sigma_n^2$ and $\sigma_z^2$ are small (upper right corner of Fig.~\ref{fig:theoretical}(b)).

While geometric compression thus aligns well with the concept of tight clusters, information-theoretic compression can have (at least) two causes: geometric compression (in the sense of tight clusters) and uninformative encoders (in the sense of strong encoder noise and/or underfitting). This resonates with the insights of \cite{kolchinskycaveats}, where it was shown that for deterministic classification problems the IB-optimal solutions do not have to be meaningful, but can be simply uninformative with a probability that depends on $\beta$.

Our example also resolves the apparent contradiction between  the negative correlation between geometric and information-theoretic compression observed in Section~\ref{sec:experiments}, and the positive correlation observed in the literature.
The correlation is negative if for a fixed spread $\sigma_z^2$ of the deterministic encoder function the added noise increases.
For example, in the CEB framework, the encoder maps inputs $X$ to latent representations $Z$ via a stochastic mapping with additive Gaussian noise, i.e., $Z = \mu(X) + D(X)$, where $D(X) \sim \mathcal{N}(0, \sigma^2(X) I)$. Stronger regularization in CEB (via a larger trade-off parameter) encourages the model to reduce the mutual information $I(X; Z)$ by increasing the conditional entropy $H(Z|X)$, which is practically achieved by increasing the variance $\sigma^2(X)$ of the injected noise. This leads to latent representations that are more dispersed or noisier and compress MI. In contrast, the correlation is positive if, for a fixed noise variance that allows for sufficient MI, the deterministic encoder is varied (as in~\cite{goldfeld2019estimating}). 

Since for a trainable stochastic encoder neither $\sigma_n^2$ nor $\sigma_z^2$ can change in isolation, the actual picture will be even more nuanced than described here.

\end{document}